\documentclass[nohyperref]{article}

\usepackage{microtype}
\usepackage{graphicx}
\usepackage{subfigure}
\usepackage{booktabs} 

\usepackage{hyperref}

\usepackage[noend]{algorithmic}

\usepackage[accepted]{icml2023}

\usepackage{amsmath}
\usepackage{amssymb}
\usepackage{mathtools}
\usepackage{amsthm}
\usepackage{multirow}
\usepackage{listings}
\usepackage{enumitem}

\usepackage[scale=0.85]{sourcecodepro}
\usepackage[capitalize,noabbrev]{cleveref}

\theoremstyle{plain}
\newtheorem{theorem}{Theorem}

\newtheorem{definition}{Definition}

\theoremstyle{remark}

\usepackage[textsize=tiny]{todonotes}

\tikzset{elegant/.style={smooth, black, thick, samples=101}}
\tikzstyle{round}=[circle ,text centered,draw=black]
\tikzstyle{arrow} = [-,>=stealth, thick]
\tikzstyle{rct}=[rectangle,draw,thin,fill=white]
\usetikzlibrary{shapes.geometric,arrows,decorations.pathmorphing,backgrounds,positioning,fit,matrix,calc}

\newcommand{\dan}[1]{{\footnotesize[\color{red}DS: #1]}}
\newcommand{\javier}[1]{{\footnotesize[\color{teal}jb: #1]}}
\newcommand{\jinlin}[1]{{\footnotesize[\color{violet}jl: #1]}}

\renewcommand{\dan}[1]{}
\renewcommand{\javier}[1]{}
\renewcommand{\jinlin}[1]{}
\newcommand{\edited}[1]{{\color{black}#1}}

\newcommand{\x}{\mathbf{x}}
\newcommand{\pa}{{\normalfont\text{pa}}}
\newcommand{\pred}{\text{pred}}
\newcommand{\HalfCauchy}{\text{HalfCauchy}}

\newcommand{\AFFINE}{\text{AFFINE}}
\newcommand{\AFFINEALL}{\text{AFFINE\_ALL}}
\newcommand{\AFFINECOEFF}{\text{AFFINE\_COEFF}}

\newcommand{\LINEAR}{\text{LINEAR}}
\newcommand{\DEPENDENT}{\text{DEPENDENT}}
\newcommand{\CONJUGATE}{\text{CONJUGATE}}
\newcommand{\REVERSE}{\text{REVERSE}}
\newcommand{\MARGINALIZE}{\text{MARGINALIZE}}
\newcommand{\RECOVER}{\text{RECOVER}}

\newcommand{\X}{\mathcal{X}}
\newcommand{\given}{\,{|}\,}

\makeatletter
\newcommand{\newreptheorem}[2]{\newtheorem*{rep@#1}{\rep@title}\newenvironment{rep#1}[1]{\def\rep@title{#2 \ref*{##1}}\begin{rep@#1}}{\end{rep@#1}}}
\makeatother
\newreptheorem{theorem}{Theorem}

\icmltitlerunning{Automatically marginalized MCMC}

\begin{document}

\twocolumn[
\icmltitle{Automatically Marginalized MCMC in Probabilistic Programming}




\begin{icmlauthorlist}
\icmlauthor{Jinlin Lai}{sch}
\icmlauthor{Javier Burroni}{sch}
\icmlauthor{Hui Guan}{sch}
\icmlauthor{Daniel Sheldon}{sch}
\end{icmlauthorlist}

\icmlaffiliation{sch}{University of Massachusetts Amherst}

\icmlcorrespondingauthor{Jinlin Lai}{jinlinlai@cs.umass.edu}

\icmlkeywords{Machine Learning, ICML}

\vskip 0.3in
]



\printAffiliationsAndNotice{} 

\begin{abstract}
	Hamiltonian Monte Carlo (HMC) is a powerful algorithm to sample latent variables from Bayesian
	models. 
	The advent of probabilistic programming	languages (PPLs) frees users from writing inference algorithms and lets users focus on modeling. 
	However, many models are difficult for HMC to solve directly, and often require tricks like model reparameterization. 
	We are motivated by the fact that many of those models could be simplified by marginalization. 
	We propose to use automatic marginalization as part of the sampling process using HMC in a graphical model extracted from a PPL, which substantially improves sampling from real-world hierarchical models.
\end{abstract}
\javier{I don't like to write graph-based PPL when referring to pyro (in the abstract).}
\dan{Possible alternative: ``graphical-model extracted from a PPL'' }
\setcounter{footnote}{1} 
\section{Introduction}

Probabilistic programming languages (PPLs) promise to automate Bayesian reasoning.
A user specifies a probabilistic model and provides data, and the PPL automatically performs inference to approximate the posterior distribution.
The user derives scientific insights without highly specialized expertise in probabilistic inference~\cite{van2018introduction}.
Through tools like BUGS~\cite{lunn2009bugs}, JAGS~\cite{hornik2003jags}, and Stan~\cite{carpenter2017stan}, this paradigm has had tremendous impact in the applied sciences, and there has been considerable research in computer science to advance the foundations of PPLs \cite{GoodmanMRBT08,wood-aistats-2014, dippl,Cusumano-Towner:2019:GGP:3314221.3314642, InferNET18}. 

PPLs vary in many dimensions, including the distributions they can represent and their primary inference approach. We focus on a setting that has had large impact in practice, where
the program corresponds to a graphical model and is compiled to a differentiable log-density function for inference by a variant of Hamiltonian Monte Carlo (HMC)~\cite{duane1987hybrid,neal2012bayesian}.

Despite their promise, the barrier between users and inference in PPLs is often blurred.
There may be different ways to write a model, with inference performance depending critically on the specific choice, such that users again need specialized knowledge.
One issue is the main focus of this paper: it is often possible to analytically marginalize some variables from the model so the inference method operates on a smaller model, which can lead to substantial performance gains.
For example, this idea is used in collapsed Gibbs sampling~\cite{liu1994collapsed}. 
However, such reformulations are almost always done manually and place a significant burden on the user (see examples in Figure~\ref{fig:eight_schools_procedure} and Section~\ref{sec:motivating_examples}).

We develop a method to automatically marginalize variables in a user-specified probabilistic program for inference with HMC. 
Our method works by first compiling a probabilistic program into a graphical model.
Although most HMC-based PPLs compile directly to a log-density, we use the program-tracing features of JAX~\cite{jax2018github} to extract a graphical-model representation of programs written in NumPyro.
In the graphical model, we identify \edited{local} conjugacy relationships that allow \edited{edges to be reversed}. 
\edited{Then we manipulate the graphical model to create unobserved leaf variables, such that they can be marginalized.}
HMC is run on the reduced model, and the marginalized variables are recovered by direct sampling conditional on the variables sampled by HMC.
\edited{For models that have no conjugacy relationships, our method reduces to vanilla HMC.}
Importantly, the interface between the user and the PPL does not change.

Experiments show that our methods can substantially improve the effectiveness of samples from hierarchical partial pooling models and hierarchical linear regression models and significantly outperforms model reparameterization~\cite{betancourt2015hamiltonian} in those models where both apply.
Our implementation is limited to scalar and elementwise array operations and may require user input to avoid excessive JAX compilation times, though these limitations are not fundamental.
\edited{Our code is available at \href{https://github.com/lll6924/automatically-marginalized-MCMC}{https://github.com/lll6924/automatically-marginalized-MCMC}.}

\subsection{Scope and relation to prior work}


As stated above, we focus on the restricted class of probabilistic programs for which we can extract a graphical model representation.
In particular, we focus on (1) \emph{generative} PPLs,\footnote{See \citet{baudart2021compiling} for a definition.} where a model is expressed by writing a sampling procedure, and (2) programs that correspond to a directed graphical model, which means that random variables are generated according to a fixed sequence of conditional distributions in each program execution \edited{(without any stochastic branches or unbounded loops)}.
This includes most applied statistical models written in generative PPLs such as Pyro~\cite{bingham2018pyro}, NumPyro~\cite{phan2019composable}, PyMC~\cite{patil2010pymc}, Edward~\cite{tran2017deep} and TensorFlow Probability~\cite{piponi2020joint}.
It does not directly include Stan programs, which do not always specify a sampling procedure, though most can be converted to do so~\cite{baudart2021compiling}.

Our work builds on prior research on automatic marginalization~\cite{hoffman2018autoconj} and shares technical underpinnings with work to automatically Rao-Blackwellize particle filters for evaluation-based PPLs~\cite{MurrayLKBS18,AtkinsonSEMI22}.
\edited{
  The key distinction of our work is that we leverage the graphical model structure to reformulate a model in a fully automatic way for inference with HMC. 
  Autoconj~\cite{hoffman2018autoconj} achieves the same mathematical goal of integrating out latent variables by recognizing patterns of conjugacy in a log-density function.
  It provides primitives for marginalizing individual variables and computing complete conditional distributions, but not an overall routine to reformulate a model; users must select how to apply the primitives to transform a model or perform inference. 
  With our pipeline, users only need to provide the model, and an end-to-end algorithm (Algorithm~\ref{algorithm1}) utilizes the graph structure to marginalize variables in a fully automatic way. 
  \citet{MurrayLKBS18,AtkinsonSEMI22} perform local model transformations that are similar to ours within a sampling procedure for particle-filter based inference, which applies to very general probabilistic programs.
  In contrast, we focus on programs that correspond to graphical models, which allows us to perform marginalization prior to inference with a fully symbolic representation, and gives a generally useful model reformulation that can be used with many inference algorithms.
  
  A longer discussion of related work appears in Section~\ref{sec:related}.  
}

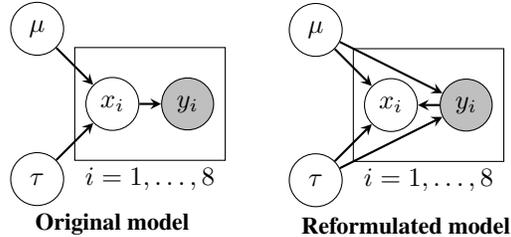
\begin{figure}[t]
\centering
  \begin{minipage}{.45\linewidth}
  \begin{center}
		\begin{tikzpicture}
			\node[round,xshift=0cm, yshift=0cm, label=center:$\mu$](mu){\phantom{$x_i$}};
			\node[round,xshift=0cm, yshift=-2cm, label=center:$\tau$](tau){\phantom{$x_i$}};
			\node[round,xshift=1cm, yshift=-1cm](xi){$x_i$};
			\node[round,xshift=2cm, yshift=-1cm,fill=lightgray](yi){$y_i$};
			\draw[draw=black] (0.5,-0.25) rectangle ++(2,-1.5);
			\node at (1.5,-2) {$i=1,\dots,8$};
			\draw [arrow, ->] (mu) -- (xi);
			\draw [arrow, ->] (tau) -- (xi);
			\draw [arrow, ->] (xi) -- (yi);
			\node at (1,-2.6) {\small \bf Original model};
		\end{tikzpicture}
		\end{center}
  \end{minipage}
  \begin{minipage}{0.45\linewidth}
  \begin{center}
		\begin{tikzpicture}
			\node[round,xshift=0cm, yshift=0cm, label=center:$\mu$](mu){\phantom{$x_i$}};
			\node[round,xshift=0cm, yshift=-2cm, label=center:$\tau$](tau){\phantom{$x_i$}};
			\node[round,xshift=1cm, yshift=-1cm](xi){$x_i$};
			\node[round,xshift=2cm, yshift=-1cm,fill=lightgray](yi){$y_i$};
			\draw[draw=black] (0.5,-0.25) rectangle ++(2,-1.5);
			\node at (1.5,-2) {$i=1,\dots,8$};
			\draw [arrow, ->] (mu) -- (yi);
			\draw [arrow, ->] (tau) -- (yi);
			\draw [arrow, ->] (yi) -- (xi);
			\draw [arrow, ->] (mu) -- (xi);
			\draw [arrow, ->] (tau) -- (xi);
			\node at (1.2,-2.6) {\small \bf Reformulated model};
		\end{tikzpicture}
		\end{center}
  \end{minipage}
  \caption{Graphical models of original and reformulated eight schools models.
    We use a plate to represent shared substructure of different branches. Gray variables are observed.}
\label{fig:eight_schools_procedure}
\end{figure}

\section{Motivating examples}

\label{sec:motivating_examples}

We first present example models where marginalization can significantly benefit HMC-based inference. 


\textbf{The eight schools model} \cite{gelman1995bayesian} is an important demonstration model for PPLs~\cite{gorinova2022thesis} and reparameterization \cite{papaspiliopoulos2007general}.
It is a hierarchical model to study the effect of coaching on SAT performance in eight schools.
An example probabilistic program for eight schools with NumPyro-like syntax is:
\begin{lstlisting}[language=Python,morekeywords={with},basicstyle=\linespread{0.9}\footnotesize\ttfamily]
def eight_schools(sigma, y):
    mu = sample(normal(0, 5))
    tau = sample(half_cauchy(5))
    with plate(8):
    	x = sample(normal(mu, tau))
    	observe(normal(x, sigma), y)
\end{lstlisting}
Mathematically, the model is
\begin{align}
	\mu&\sim\mathcal{N}(0,5^2),\ \ \tau\sim \HalfCauchy(5),\notag\\
	x_i&\sim\mathcal{N}(\mu,\tau^2),\ y_i\sim\mathcal{N}(x_i,\sigma_i^2),\notag
\end{align}
where $i\in\{1,\dots,8\}$ and $(\sigma_{1:8},y_{1:8})$ are given as data. We want to reason about all latent variables, $\mu$, $\tau$ and $x_{1:8}$.
A PPL will compile the model code to a log joint density $\log p(\mu,\tau,x_{1:8},y_{1:8})$ and then run HMC over the latent variables $\mu$, $\tau$ and $x_{1:8}$.\footnote{In practice, latent variables are transformed to have real support~\cite{KucukelbirTRGB17}.}
However, there is another model with the same joint density:
\begin{align}
	\mu&\sim\mathcal{N}(0,5^2),\quad\quad\quad\tau\sim \HalfCauchy(5),\notag\\
	y_i&\sim\mathcal{N}(\mu,\tau^2+\sigma_i^2),\ x_i\sim\mathcal{N}\left(\frac{y_i\tau^2+\mu\sigma_i^2}{\tau^2+\sigma_i^2},\frac{\tau^2\sigma_i^2}{\tau^2+\sigma_i^2}\right).\notag
\end{align}
Both models are shown as graphical models in Figure~\ref{fig:eight_schools_procedure}:  they have different causal interpretations but identical joint distributions and are therefore the same for performing inference.
Importantly, in the reformulated model, we no longer need to run HMC over all latent variables.
Since only $y_{1:8}$ are observed, it is possible to marginalize $x_{1:8}$ to obtain the reduced model $p(\mu,\tau,y_{1:8})=p(\mu)p(\tau)\prod_{i=1}^8p(y_i\given \mu,\tau)$.
We can sample $\mu$ and $\tau$ by running HMC on the reduced model then sample $x_{1:8}$ directly from $p(x_{1:8}\given \mu,\tau,y_{1:8})$ given $\mu$ and $\tau$.
With this strategy, HMC samples 2 variables instead of 10, which significantly speeds up inference.

The principle that allows us to transform the model is conjugacy~(see Section 3.3.2 of \citet{murphy2012machine}).
In a Bayesian model $p(x, y) = p(x)p(y \given x)$ the prior $p(x)$ is conjugate to the likelihood $p(y \given x)$ if the posterior $p(x \given y)$ is in the same parametric family as $p(x)$ for all $y$.
For our working definition, we assume the parametric families of the prior and likelihood have a tractable density function and sampling procedure and that there is an analytical formula for the parameters of the posterior in terms of $y$.
Given these assumptions, it is also possible to sample from the marginal $p(y)$ and compute its density efficiently.\footnote{To sample, draw $x \sim p(x), y \sim p(y \given x)$ and ignore $x$; for the density, use $p(y) = \frac{p(x_0)p(y \given x_0)}{p(x_0 \given y)}$ for a reference value $x_0$.}
Conjugacy is formally a property of distribution families, but we will also say ``$x$ is conjugate to $y$'' when the meaning is clear from context.

In the eight schools model, $x_i$ is conjugate to $y_i$ given $\mu$ and $\tau$, 
which leads to analytical expressions for the distributions $p(x_i\given\mu,\tau,y_i)$ and $p(y_i\given\mu,\tau)$ in the reformulated model and ensures they have tractable densities and samplers.

\textbf{Hierarchical linear regression.}
\label{sec:small_electric_company}
\begin{figure*}
\centering
  \begin{minipage}{.20\linewidth}
  \centering
  {\small \bf Original model}
\begin{align*}
	\log\sigma&\sim\mathcal{N}(0,1),\\
	\mu_a&\sim\mathcal{N}(0,1),\\
	a&\sim\mathcal{N}(100\mu_a,1),\\
	b_{1:2}&\sim\mathcal{N}(0,100^2),\\
	y_i&\sim\mathcal{N}(a+b_it_i,\sigma^2).
\end{align*}
	\end{minipage}
	  \begin{minipage}{.78\linewidth}
	  \centering
	  {\small \bf Distributions of $a$, $b_1$, and $b_2$, in the reformulated model}
	  {\small
	  \begin{align*}
			a&\sim \mathcal{N}\left(\frac{100\mu_a\sigma^2+y_1+y_2-b_1t_1-b_2t_2}{2+\sigma^2},\frac{\sigma^2}{2+\sigma^2}\right),\\
	b_1&\sim \mathcal{N}\left(\frac{100^2t_1(y_1-100\mu_a)}{1+100^2t_1^2+\sigma^2},\frac{100^2+100^2\sigma^2}{1+100^2t_1^2+\sigma^2}\right),\\
	b_2&\sim \mathcal{N}\left(\frac{100^2t_2(y_2+\sigma^2y_2-y_1-100\sigma^2\mu_a+b_1t_1)}{100^2t_2^2+100^2\sigma^2t_2^2+2\sigma^2+\sigma^4},\frac{2*100^2\sigma^2+100^2\sigma^4}{100^2t_2^2+100^2\sigma^2t_2^2+2\sigma^2+\sigma^4}\right).
\end{align*}
}%
	\end{minipage}
	  \caption{The original model and the reformulated model of the simplified electric company model. $(t_1,y_1),(t_2,y_2)$ are given as data.
          }
	\label{figure:electric_company_model}
\end{figure*}
The eight schools model is very simple, but already requires user effort to reformulate.
To emphasize the complexity of reformulating larger models, in Figure \ref{figure:electric_company_model} we present a simplified version of the electric company model \cite{gelman2006data}.
The full model appears in Section \ref{exp:hlr}.
The observed variables are $y_1$ and $y_2$.
One can guess that the model can be reformulated because, conditioned on $\sigma$, all variables are normal with means that are affine functions of other variables.
However, the calculations are complex: the right side of Figure~\ref{figure:electric_company_model} shows a portion of the reformulated model.
In this version, HMC is run to sample $\mu_a$ and $\sigma$  (distributions not shown) conditioned on $y_1$ and $y_2$.
Then $a$, $b_1$, and $b_2$ are reconstructed conditioned on $\mu_a, \sigma, y_1, y_2$ by sampling from the shown distributions. 
By reducing the number of variables from 5 to 2, HMC inference can be accelerated.
However, it is extremely cumbersome for the user to derive the new model, which no longer corresponds to the originally conceived data generating process. 
We wish to automate this procedure so users only write the original model and our framework reformulates it.

\section{Automatically marginalized MCMC}

Given a program written by a user, our method will construct a graphical model and then manipulate it into a reformulated model for which MCMC samples fewer variables.
The key operation will be \emph{reversing} certain edges (based on conjugacy) to create unobserved leaf nodes that can be marginalized.
For example, in the eight schools model of Figure \ref{fig:eight_schools_procedure}, the edge from $x_i$ to $y_i$ is reversed (which has the side effect of creating edges from $\mu$ and $\tau$ to $y_i$), after which $x_i$ is a leaf.
In this section, we first develop the algorithm assuming a suitable graphical model representation.
Then, in Section \ref{sec:implementation}, we describe how we obtain such a representation in our implementation using JAX and NumPyro.


\subsection{Graphical model representation}
Assume there are $M$ random variables  $x_1, x_2,\dots,x_M$ where $x_i$ belongs to domain $\X_i$.
The full domain is $\X = \prod_i \X_i$.
For a set of indices $A$, we write $\x_A = (x_i)_{i \in A}$ and $\X_A = \prod_{i \in A} \X_i$. 
A graphical model $G$ is defined by specifying a distribution family for each node together with a mapping from parents to parameters.
Specifically, for node $i$, let $D_i$ represent its distribution family from a finite set of options (e.g., ``Normal'', ``Beta'', etc.), let $\pa(i) \subseteq \{1, \ldots, M\}$ be its parents, and let $f_i\colon \X_{\pa(i)} \to \Theta_i$ be a mapping such that $x_i$ has distribution $D_i(\theta_i)$ with parameters $\theta_i = f_i(\x_{\pa(i)})$.
For example, if $x_2\sim\mathcal{N}(x_1,1)$, then $D_2=\text{``Normal''}$, $\pa(2)=\{1\}$, and $f_2(x_1)=(x_1,1)$. Furthermore, for each distribution family, assume a density function and sampling routine are available. Let $p_i(x_i \given  \theta_i)$ be the density function for node $i$ and $h_i(u \given  \theta_i)$ be the sampling function, which maps a random seed $u$ to a sample from $D_i(\theta_i)$.
The parent relationship is required to be acyclic. 
Initially, nodes will be ordered topologically so that $\pa(i) \subseteq \{1, \ldots, i-1\}$.
Our algorithms will manipulate the graphical model to maintain acyclicity but will not preserve the invariant that nodes are numbered topologically.
In our example models we use standard notation for hierarchical models with variable names such as $\mu, \tau, x_i, y_i$; in these cases, the mapping to a generic sequence of random variables $x_1, \ldots, x_M$, the parent relationship, and the distribution families are clear from context.



With this representation, given concrete values of all variables, the log density can be computed easily as $\sum_{i=1}^M \log p_i\big(x_i \given f_i(\x_{\pa(i)})\big)$, assuming nodes are ordered topologically.
Generating a joint sample is similar: iterate through  nodes and sample $x_i = h_i\big(u \given f_i(\x_{\pa(i)})\big)$.
A key idea of our approach is that by factoring the log-density computation into the sequence of conditional functions $f_i$ for each random variable, we can manipulate the conditional distributions to achieve automatic marginalization.

\subsection{Computation graph representation}
\label{sec:computation_graph}
Our operations to transform the graphical model will require examining and manipulating the functions $f_i(\x_{\pa(i)})$ mapping parents to distribution parameters.
For example, in the electric company model of Figure~\ref{figure:electric_company_model}, we need to detect from the symbolic expression $y_2 \sim \mathcal{N}(a + b_2 t_2, \sigma^2)$ that the mean parameter is an affine function of $b_2$, which is required to reverse the edge $b_2 \to y_2$.
Similarly, we must manipulate symbolic expressions to obtain ones like those in the reformulated model.
For this purpose we assume functions are represented as computation graphs. 

Consider an arbitrary function $f(x_{i_1}, x_{i_2}, \ldots, x_{i_k})$ for $i_1, \ldots, i_k \in \{1, \ldots, M\}$.
We assume the computation graph of $f$ is specified as a sequence of $N_f$ \emph{primitive operations} that each write one value~\cite{griewank2008evaluating}, which is similar to the \emph{JAX expression (Jaxpr)} representation we can obtain from JAX.
Specifically, the sequence of values $w_1, w_2, w_3, \ldots, w_{k+N_f}$ are computed as follows: (1) the first $k$ values are the inputs to the function, i.e., $w_j = x_{i_j}$ for $j = 1$ to $k$, and (2) each subsequent value is computed from the preceding ones as $w_j=\phi_j(\mathbf w_{\pred(j)})$, where $\phi_j$ is a primitive operation (e.g., ``ADD'', ``MUL'', ``SQUARE'') on values $\mathbf w_{\pred(j)}$ and $\pred(j) \subseteq \{1,\ldots, j-1\}$ is the set of predecessors of $j$. 
The predecessor relationship defines a DAG for the variables in a computation graph.

We will also need to algorithmically manipulate computation graphs.
In the text, we will denote manipulations symbolically as follows.
Suppose $f(\x_\text{A})$ and $g(\x_\text{B})$ are two functions represented by computation graphs with potentially overlapping sets of input variables.
We use expressions such as $f*g$ or $f+g$ to mean the new computation graph representing this symbolic expression.
For example the computation graph for $f+g$ has input variables $\x_{\text{A}\cup \text{B}}$ and consists of the graphs for $f$ and $g$ together with one additional node (primitive operation) for the final addition.

\subsection{Marginalizing unobserved leaf nodes}

As a first useful transformation of the graphical model, we consider how to improve HMC if there is an unobserved leaf node. 
Without loss of generality, assume the leaf is numbered $M$.
Then we can factor the joint distribution as
\begin{equation}
  \label{eq:leaf_factorization}
  p(\x_{1:M}) = p(\x_{1:M-1}) p(x_M \given  \x_{1:M-1}),
  \end{equation}
and run HMC on the marginalized model $p(\x_{1:M-1})$, then sample $x_M$ directly from $p(x_M\given \x_{1:M-1})$ by executing $h_M(\,\cdot\,\given \,f_M(\x_{\pa(M)}))$.
Importantly, the marginal $p(\x_{1:M-1})=\prod_{i=1}^{M-1} p_i(x_i \given  f_i(\x_{\pa(i)}))$ is simply the original graphical model with the leaf node deleted, so it is tractable.
More generally, the argument is easily extended by repeatedly stripping leaves to marginalize all variables with no path to an observed variable for HMC, then to reconstruct those variables by \emph{forward sampling} (sometimes called \emph{ancestral sampling}) from the graphical model~\cite{koller2009probabilistic}.


\subsection{Marginalizing non-leaf nodes by edge reversals}
\label{sec:marginalize_reversal}

Generative models such as the eight schools and electric company models do not have unobserved leaf nodes in their original forms, since these nodes would play no useful role in the data-generating process.
Instead, our goal will be to transform the model by a sequence of \emph{edge reversals} to create unobserved leaf nodes.
Each edge reversal will preserve the joint distribution of the graphical model, so it is the same for performing inference.
However, it will not preserve the causal semantics of the data-generating process (which is not required for inference), so it is reasonable for the transformed model to have unobserved leaf nodes.


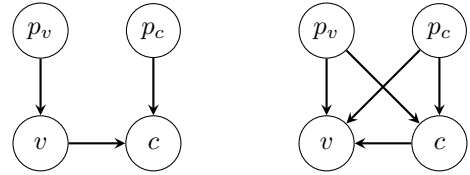
\begin{figure}
\centering
  \begin{minipage}{.45\linewidth}
  \begin{center}
		\begin{tikzpicture}
			\node[round,xshift=0cm, yshift=0cm](pav){$p_v$};
			\node[round,xshift=0cm, yshift=-1.5cm, label=center:$v$](v){\phantom{$p_v$}};
			\node[round,xshift=1.5cm, yshift=0cm](pac){$p_c$};
			\node[round,xshift=1.5cm, yshift=-1.5cm, label=center:$c$](c){\phantom{$p_v$}};
			\draw [arrow, ->] (pav) -- (v);
			\draw [arrow, ->] (pac) -- (c);
			\draw [arrow, ->] (v) -- (c);
		\end{tikzpicture}
		\end{center}
  \end{minipage}
  \begin{minipage}{0.45\linewidth}
  \begin{center}
		\begin{tikzpicture}
			\node[round,xshift=0cm, yshift=0cm](pav){$p_v$};
			\node[round,xshift=0cm, yshift=-1.5cm, label=center:$v$](v){\phantom{$p_v$}};
			\node[round,xshift=1.5cm, yshift=0cm](pac){$p_c$};
			\node[round,xshift=1.5cm, yshift=-1.5cm, label=center:$c$](c){\phantom{$p_v$}};
			\draw [arrow, ->] (pav) -- (v);
			\draw [arrow, ->] (pac) -- (c);
			\draw [arrow, ->] (c) -- (v);
			\draw [arrow, ->] (pav) -- (c);
			\draw [arrow, ->] (pac) -- (v);
		\end{tikzpicture}
		\end{center}
  \end{minipage}
  \caption{Reversing edge $v\to c$. Nodes $p_v\in\pa(v)$ and $p_c\in \pa(c)\setminus\{v\}$ are representative parents to demonstrate the transformation. Left: the graphical model before reversing $v\to c$. Right: the graphical model after reversing $v\to c$.}
\label{fig:reversing}
\end{figure}

\textbf{Reversing a single edge.}
The process of reversing a single parent-child edge $v \to c$ is illustrated in Figure~\ref{fig:reversing}.
There must be no other path from $v$ to $c$; otherwise reversing the edge would create a cycle. 
In the example, there is no other path because $v$ has only one child.
Let us define the ``local distribution'' of $x_v$ and $x_c$ as the product of the conditional distributions of those two variables given their parents, which looks like $p(x_v\given \cdots )p(x_c\given x_v, \cdots)$.
If these distributions satisfy the appropriate conjugacy relationship, we can derive replacement factors that look like $p(x_c \given  \cdots) p(x_v \given  x_c, \cdots)$ to ``reverse'' the $v \to c$ edge while preserving the local distribution.
Formally, the operation is:
\begin{definition}[Edge reversal]
\label{def:edge_reversal}
Assume $G$ is a graphical model where node $v$ is a parent of $c$ and there is no other path from $v$ to $c$. Reversing edge $v \to c$ replaces factors $p(x_v \given  \x_{\pa(v)}) p(x_c \given  x_v, \x_{\pa(c)\setminus \{v\}})$ by $p(x_c \given  \x_U) p(x_v\given x_c, \x_U)$ and updates the parent sets as $\pa'(c) = U$, $\pa'(v) = U \cup \{c\}$ for $U=\pa(v)\cup \pa(c)\setminus\{v\}$.
\end{definition}
It is easy to show that edge reversal yields a graphical model with the same joint distribution as the original.
\edited{A proof appears in Appendix \ref{proof:same_joint}.} \dan{TODO: Jinlin, please modify the previous statement to match what appears.}
To understand the utility of this operation, observe in Figure~\ref{fig:reversing} that node $v$ becomes a leaf and can be marginalized after reversing $v \to c$.
In principle, any edge can be reversed, but it is only tractable when one can derive the replacement factors.
We can do so if the distributions are \emph{locally conjugate}:
\begin{definition}[Local conjugacy]
  Let $G$ be a graphical model where node $v$ is a parent of $c$.
  We say the distribution of $x_v$ is \emph{locally conjugate} to the distribution of $x_c$ if $\hat p(x_v) := p(x_v \given \x_{\pa(v)})$ is conjugate to $\hat p(x_c \given x_v) := p(x_c \given x_v, \x_{\pa(c) \setminus \{v\}})$ for all values of $\x_{\pa(v)}$ and $\x_{\pa(c) \setminus \{v\}}$.
\end{definition}

For example, in the model $x_1\sim\mathcal{N}(0,1)$, $x_2\sim\mathcal{N}(x_1,1)$, $x_3\sim\mathcal{N}(x_1x_2,1)$,
the random variable $x_1$ is conjugate to $x_2$, and also conjugate to $x_3$ for fixed $x_2$.
So it is locally conjugate to both $x_2$ and $x_3$.
Section~\ref{sec:reversal} will describe the details of edge reversal using our graphical model representation and specific conjugate pairs of distribution families.
\edited{
  Local conjugacy should not be confused with the related concept of \emph{conditional conjugacy}~\cite{murphy2012machine}, which occurs when the complete conditional distribution $p(x_v \given \x_{\{1,\ldots,M\} \setminus \{v\}})$ is in the same family as the generative distribution $p(x_v \given \x_{\pa(v)})$. Conditional conjugacy describes the relationship of a single variable to the complete distribution, while local conjugacy is a pairwise relation. 
  }

\begin{figure}
\centering
  \begin{minipage}{.32\linewidth}
  \begin{center}
		\begin{tikzpicture}
			\node[round,xshift=0cm, yshift=0cm, label=center:$v$](v){\phantom{$c_1$}};
			\node[round,xshift=-0.75cm, yshift=-1.5cm](c1){$c_1$};
			\node[round,xshift=0.75cm, yshift=-1.5cm](c2){$c_2$};
			\draw [arrow, ->] (v) -- (c1);
			\draw [arrow, ->] (v) -- (c2);
			\draw [arrow, ->] (c1) -- (c2);
		\end{tikzpicture}
		\end{center}
  \end{minipage}
  \begin{minipage}{.32\linewidth}
  \begin{center}
		\begin{tikzpicture}
			\node[round,xshift=0cm, yshift=0cm, label=center:$v$](v){\phantom{$c_1$}};
			\node[round,xshift=-0.75cm, yshift=-1.5cm](c1){$c_1$};
			\node[round,xshift=0.75cm, yshift=-1.5cm](c2){$c_2$};
			\draw [arrow, ->] (v) -- (c1);
			\draw [arrow, ->,red] (c2) -- (v);
			\draw [arrow, ->] (c1) -- (c2);
		\end{tikzpicture}
		\end{center}
  \end{minipage}
    \begin{minipage}{.32\linewidth}
  \begin{center}
		\begin{tikzpicture}
			\node[round,xshift=0cm, yshift=0cm, label=center:$v$](v){\phantom{$c_1$}};
			\node[round,xshift=-0.75cm, yshift=-1.5cm](c1){$c_1$};
			\node[round,xshift=0.75cm, yshift=-1.5cm](c2){$c_2$};
			\draw [arrow, ->,blue] (c1) -- (v);
			\draw [arrow, ->] (v) -- (c2);
			\draw [arrow, ->] (c1) -- (c2);
		\end{tikzpicture}
		\end{center}
  \end{minipage}
    \caption{Challenges of reversing multiple outgoing edges. Left: the local structure. Middle: A loop is formed after reversing $v\to c_2$, which is invalid.
      Right: The model is still valid after reversing $v\to c_1$, and now $v\to c_2$ can also be reversed.}
\label{fig:reversing_multiple}
\end{figure}
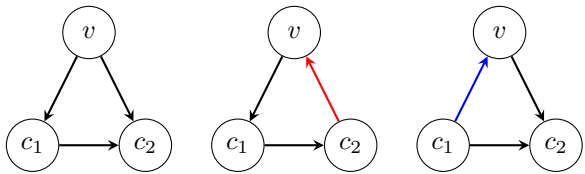

\textbf{Creating a leaf by reversing all outgoing edges of a node.}
We next consider how, if possible, to convert an arbitrary node $v$ to a leaf by reversing all of its outgoing edges.
Suppose $v$ has $H$ children $c_1,\dots,c_H$.
For an arbitrary child $c$ it may not be safe to reverse $v \to c$ even if $x_v$ is locally conjugate to $x_{c}$ because there may be another path from $v$ to $c$, which would lead to a cycle.
See Figure \ref{fig:reversing_multiple} for an example.
However, if $c$ is minimal among $c_1,\dots,c_H$ in a topological ordering of $G$, then there can be no other $v \to c$ path, so it is safe to reverse $v \to c$.
Further, after reversing the edge, $v$ will move in the topological ordering to appear after $c$ but before the other children, and the relative ordering of the other children will remain unchanged.
Then another child will be minimal in the topological ordering.
Therefore, if it is possible to convert $v$ to a leaf, we should reverse the edges from $v$ to each of its children following their topological ordering.
This reasoning is summarized in the following theorem, which is proved in Appendix \ref{proof:th1}.
\begin{theorem}
\label{thm:theorem1}
Let $G$ be a graphical model where node $v$ has children $c_1,\dots,c_H$.
If $x_v$ is locally conjugate to each of $x_{c_1},\dots,x_{c_H}$, then $v$ can be turned into a leaf by reversing the edges from $v$ to each child sequentially in the order of a topological ordering of the children. 
\end{theorem}

\textbf{Marginalizing many non-leaf nodes.}
Theorem~\ref{thm:theorem1} describes how to modify a graphical model, while preserving the joint distribution, to convert one non-leaf node to a leaf so it can be marginalized.
We now wish to use this operation to marginalize as many nodes as possible.
The $\MARGINALIZE$ function in Algorithm~\ref{algorithm1} presents our heuristic for doing so: it simply applies the operation of Theorem~\ref{thm:theorem1} to attempt to marginalize every node $v$ in \emph{reverse} topological order.
This is convenient because it automatically strips all nodes with no path to an observed variable at the same time.
If $v$ can be marginalized, the reversal operations are executed and $v$ is removed from $G$ and pushed onto a stack $S$ that determines the recovery order.
\dan{Low priority: justify recovery order }
The $\RECOVER$ function augments a sample of the non-marginalized variables with direct samples of the marginalized variables.
The next sections discuss the implementations of $\CONJUGATE$ and $\REVERSE$.

\begin{algorithm}[t]
   \caption{Marginalize and recover unobserved nodes}
   \label{algorithm1}
   \begin{algorithmic}[1]
\STATE  \textbf{function} MARGINALIZE ($G$)
\begin{ALC@g}
                        \STATE Initialize stack $S$ and sort nodes so they are numbered in topological order
                        \FOR {each unobserved node $v$ in descending order}
				\IF {$\text{CONJUGATE}(G,v,c)$ for all children $c$}
					\STATE // Marginalize $v$
					\FOR {each child $c$ in ascending order}
						\STATE $G = \text{REVERSE}(G,v,c)$
					\ENDFOR
					\STATE $\text{Remove }v\text{ from }G$
					\STATE $\text{Add }v\text{ to top of }S$
				\ENDIF
			\ENDFOR
			\STATE \textbf{return} $G$, $S$\\[5pt]
            \end{ALC@g}
\STATE  \textbf{function} RECOVER $(S, x_{v}\text{ for nodes } v\text{ not removed})$
            \begin{ALC@g}
			\WHILE {$S$ is not empty}
				\STATE $v=\text{pop}(S)$
				\STATE Sample $x_v$ given $x_{pa(v)}$
			\ENDWHILE
			\STATE \textbf{return} $x_{1:M}$
            \end{ALC@g}
\end{algorithmic}
\end{algorithm}

\subsection{Conjugacy detection}
\label{sec:conjugacy_detection}
Detecting when $x_a$ is locally conjugate to $x_b$ uses the patterns listed in Table \ref{table:conjugacy_pattern}, where
(1) $\AFFINE(u,v)$ means that $u$ can be written as $u=pv+q$ for expressions $p$ and $q$ that do not contain $v$,
(2) $\DEPENDENT(u,v)$ means that there exists a path from $v$ to $u$ in the computation graph, and
(3) $\LINEAR(u,v)$ means that $u$ can be written as $u=pv$, for an expression $p$ that does not contain $v$.
\begin{table*}[t]
\caption{Patterns of conjugacy. If conditions of a row are satisfied, then the distribution of $x_a$ is locally conjugate to the distribution of $x_b$.}
\label{table:conjugacy_pattern}
\vskip 0.15in
\begin{center}
\begin{small}
\begin{sc}
\begin{tabular}{cccc}
\toprule
Distribution of $x_a$ & Distribution of $x_b$ & Condition 1 &Condition 2\\
\midrule
$\mathcal{N}(\mu_a,\sigma_a^2)$& $\mathcal{N}(\mu_b,\sigma_b^2)$&$\text{AFFINE}(\mu_b,x_a)$&$\text{not}\ \text{DEPENDENT}(\sigma_b,x_a)$\\
$\Gamma(\alpha_a,\beta_a)$& $\Gamma(\alpha_b,\beta_b)$&$\text{LINEAR}(\beta_b,x_a)$&$\text{not}\ \text{DEPENDENT}(\alpha_b,x_a)$\\
$\Gamma(\alpha_a,\beta_a)$& $\text{Exponential}(\lambda_b)$&$\text{LINEAR}(\lambda_b,x_a)$&-\\
$\text{Beta}(\alpha_a,\beta_a)$& $\text{Binomial}(n_b,p_b)$&$p_b=x_a$&$\text{not}\ \text{DEPENDENT}(n_b,x_a)$\\
$\text{Beta}(\alpha_a,\beta_a)$& $\text{Bernoulli}(\lambda_b)$&$\lambda_b=x_a$&-\\
\bottomrule
\end{tabular}
\end{sc}
\end{small}
\end{center}
\vskip -0.1in
\end{table*}
For example, the pattern in the first matches the case when $x_a$ and $x_b$ both have normal distributions, in which case we can extract expressions (computation graphs) for the parameters $\mu_b$ and $\sigma_b^2$ as the two outputs of the expression $f_b(\x_{\pa(b)})$.
The pattern further implies that if $\AFFINE(\mu_b,x_a)$ is true and $\DEPENDENT(\sigma_b^2,x_a)$ is false, then $x_a$ is locally conjugate to $x_b$.
The functions $\AFFINE$, $\LINEAR$ and $\DEPENDENT$ require examining computation graphs; details and pseudocode can be found in appendix \ref{sec:detail_conjugacy}.
\edited{While we focus on five common patterns, it is possible to introduce other patterns\footnote{For more background and patterns of conjugacy, please refer to \url{https://en.wikipedia.org/wiki/Conjugate_prior}.} to our pipeline, and may be possible to utilize automated conjugacy detection based on properties of exponential families as done in autoconj~\cite{hoffman2018autoconj}.}

\subsection{Edge reversal details}
\label{sec:reversal}
If conjugacy is detected, Algorithm \ref{algorithm1} will call the $\REVERSE$ operation to reverse an edge. 
Algorithm \ref{algorithm2} shows the portion of the $\REVERSE$ algorithm for normal-normal conjugacy.
We emphasize that operations like $+$, $-$ and $*$ are symbolic operations on computation graphs.
The algorithm implements well known Gaussian marginalization and conditioning formulas.
\dan{Would be nice to give mathematical formulas in typical format for comparison.}
Line~5 extracts the symbolic expressions for the parameters of the normal distributions.
Line~6 extracts expressions $p$ and $q$ such that $\mu_c = p x_v + q$; conjugacy detection has already determined that such expressions exist. \edited{The function AFFINE\_COEFF is defined in Appendix \ref{sec:detail_reversing}.}
Lines~7--12 compute symbolic expressions for parameters of the marginal $p(x_c | \cdots)$ and conditional $p(x_v | x_c, \cdots)$ and write them to $f_c$ and $f_v$.
Finally, Lines~15--16 update the DAG to reflect the new dependencies.
\begin{algorithm}[t]
   \caption{Reversing an edge (normal-normal case)}
   \label{algorithm2}
\begin{algorithmic}[1]
\STATE  \textbf{function} REVERSE ($G=(D_i,\pa(i),f_i)_{i=1}^M$,$v$,$c$)
       \begin{ALC@g}
       		\STATE // $\pa(v), \pa(c), f_v, f_c, D_c$ updated in place
                \STATE // all variables represent symbolic expressions\\[2pt]
			\IF {$D_c$ is normal \textbf{and} $D_v$ is normal}
			\STATE Let $\mu_v$ and $\sigma_v^2$ be the two output expressions $f_v$, and $\mu_c$ and $\sigma^2_c$ be the output expressions of $f_c$.
			\STATE $p,q=\text{AFFINE\_COEFF}(\mu_c,x_v)$
			\STATE $k=\sigma_v^2p/(p^2\sigma_v^2+\sigma_c^2)$
			\STATE $\mu_c'=p\mu_v+q$
			\STATE $\sigma_c'^2 = p^2\sigma_v^2+\sigma_c^2$
			\STATE $\mu_v'=\mu_v+k(x_c-\mu_c')$
			\STATE $\sigma_v'^2 = (1-kp)\sigma_v^2$
			\STATE $f_c=(\mu_c',\sigma_c'^2)$, $f_v=(\mu_v',\sigma_v'^2)$
			\ELSE
                        \STATE $\vdots$ \ \ \ \ \ (see full algorithm in Appendix \ref{sec:detail_reversing}) \\[3pt]
			\ENDIF
			\STATE $\pa(c)=(\pa(c)\setminus \{v\})\cup \pa(v)$	
			\STATE $\pa(v)=\pa(v)\cup \{c\}\cup \pa(c)$
			\STATE \textbf{return} $G$
       \end{ALC@g}
\end{algorithmic}
\end{algorithm}


\subsection{Implementation}
\label{sec:implementation}
We have assembled the pieces for automatically marginalized HMC.
The full pipeline is to: (1) extract a graphical model $G$ from the user's program, (2) call the MARGINALIZE function to get a marginalized model $G'$ and recovery stack $S$, (3) run HMC on $G'$, (4) for each HMC sample $\x$, call RECOVER$(S, \x)$ to sample the marginalized variables.  

Our implementation uses JAX \cite{jax2018github} and NumPyro \cite{bingham2018pyro, phan2019composable} to extract a graphical model $G$.
We use JAX tracing utilities to convert the NumPyro program to a JAX expression (Jaxpr), i.e., computation graph, for the entire sampling procedure.
The NumPyro program must use a thin wrapper around NumPyro's sample statement to register the model's random variables in the Jaxpr. 
We extract the distribution families from the NumPyro trace stack and obtain the parameter functions $f_i$ by parsing the Jaxpr to extract the partial computation mapping from parent random variables to distribution parameters.
As stated earlier, our approach is limited to programs that map to a graphical model, which means they sample from a fixed sequence of conditional distributions.
This closely matches those programs for which NumPyro can currently perform inference, because the JIT-compilation step of NumPyro inference requires construction of a static computation graph.
NumPyro's experimental control flow primitives (``scan'' and ``cond'') are not supported, and it may be difficult to do so. 
Our current implementation is limited to Jaxprs with scalar operations and elementwise array operations, though this restriction is not fundamental. 
We expect our approach is compatible with other PPLs that use computation graphs, with similar restrictions on programs. 

\edited{With the JAX-based implementation, the complexity of a marginalized computation graph is evaluated by the number of lines of Jaxprs.
  We have not proven a bound for the complexity of the computation graph after marginalization, but observe in experiments that the increase in size is mild.}

\section{Related work}
\label{sec:related}
Conjugacy and marginalization have long been important topics in probabilistic programming.
In BUGS \cite{lunn2000winbugs} and JAGS \cite{hornik2003jags}, conjugacy was used to improve automatic Gibbs sampling.
\edited{These systems use conjugacy to derive complete conditional distributions, not to marginalize variables.}
Hakaru \cite{narayanan2016probabilistic} and PSI \cite{gehr2016psi, gehr2020lpsi}
use symbolic integrators to perform marginalization for the purposes of exact inference.
We make use of information provided by graphical models to identify certain patterns, which is more efficient in large scale models.
Autoconj \cite{hoffman2018autoconj} proposes a term-graph rewriting system that can be used for marginalizing a log joint density with conjugacy.
Our approach is distinct in that we operate on the graphical model and computation graph for the generative process, as opposed to the log-density.
\edited{Also, as mentioned in the introduction, while autoconj provides primitives for marginalizing variables from a log-density, the user must also specify how they are applied, in order to marginalize variables or achieve their inferential goal.
  In contrast, we develop an end-to-end algorithm for transforming a graphical model; the user only needs to supply the original model. In the future, it may be possible to combine autoconj primitives with our graphical model approach.}

\citet{gorinova2021conditional} propose an information flow type system that could be applied to automatic marginalization of discrete random variables.
Following the exploration of more expressive PPLs, streaming models have attracted much attention.
\citet{MurrayLKBS18} proposed delayed sampling, which uses automatic marginalization to improve inference via the Rao-Blackwellized particle filter (RBPF) \cite{doucett2000rao}.
Delayed sampling has been developed in Birch \cite{murray2018automated}, Pyro \cite{bingham2018pyro} with funsors \cite{obermeyer2019functional,obermeyer2019tensor}, Anglican \cite{lunden2017delayed} and ProbZelus \cite{baudart2020reactive}. \citet{AtkinsonSEMI22} propose semi-symbolic inference, which further expands the applicability of delayed sampling to models with arbitrary structure.
\dan{changed ``reactive'' to ``streaming''}
Our work is distinct in that we statically analyze a model prior to performing inference for the purpose of improving MCMC: this makes our approach ``fully symbolic'' (no concrete values are available) and leads to different algorithmic considerations, though our Algorithm \ref{algorithm1} shares technical underpinnings with the hoisting algorithm in \citet{AtkinsonSEMI22}; see Appendix \ref{sec:relation_hoisting} for more details.

\edited{Manually reformulating a model by marginalizing variables to achieve more efficient inference is a well known trick in applied probabilistic modeling and is also referred to as \emph{collapsed sampling} in some contexts. Specifically, collapsed Gibbs sampling (CGS) \citep{liu1994collapsed} is an algorithm that marginalizes conjugate prior variables and recovers them afterwards to facilitate Gibbs sampling (GS). It can be applied to mixture models such as latent dirichlet allocation (LDA) models \citep{porteous2008fast,murphy2012machine}. Often, the remaining variables after marginalization are discrete, so it is possible to derive any conditional distributions for GS by normalization. The key distinction of our work is that the model is automatically reformulated, instead of doing so manually.
Although we focused on HMC, our methods could also be used to automatically collapse models for Gibbs sampling or other MCMC algorithms.}


There are many works that improve HMC inference in PPLs from different perspectives.
Stan \cite{carpenter2017stan} has had tremendous impact using HMC inference for PPLs.
Because Stan programs specify a log-density and not a sampling procedure, our idea does not directly apply to Stan programs. However, many Stan programs are generative in spirit, and \citet{baudart2021compiling} characterize a subset of Stan programs on which the methods of this paper can be applied directly.
\citet{papaspiliopoulos2007general} propose a general framework for non-centered (re)parameterization in MCMC. \citet{gorinova2020automatic} automate the procedure of choosing parameterizations of models using variational inference.
In \citet{parno2018transport} and \citet{hoffman2019neutra}, the parameterizations of all latent variables are learned as normalizing flows \cite{PapamakariosNRM21,rezende2015variational}.
In models where some variables are marginalizable, our method works better than reparameterization: 
see Section \ref{exp:hlr} for an example.
\citet{mak2022nonparametric} use the framework of involutive MCMC \cite{neklyudov2020involutive, cusumano2020automating} to extend the applicability of MCMC to non-parametric models in PPLs.

\section{Experiments}
We evaluate the performance of our method on two classes of hierarchical models where conjugacy plays an important role.
We use NumPyro's no-U-turn sampler (NUTS) \cite{HoffmanG14} in all experiments, denoted HMC hereafter.
Our approach is ``HMC with marginalization'' (HMC-M).
For all experiments, we use 10,000 warm up samples to tune the sampler, 100,000 samples for evaluation, and evaluate performance via effective sample size (ESS) and time (inclusive of JAX compilation time).

\subsection{Hierarchical partial pooling models}
A hierarchical partial pooling (HPP) model \cite{gelman1995bayesian} has the form
$p(\theta,z_{1:n},y_{1:n})=p(\theta)\prod_{i=1}^n p(z_i \given \theta)p(y_i\given \theta,z_i,x_i)$,
where $(x_i, y_i)$ are observed covariate and response values for the $i$th data point, $z_i$ is a local latent variable, and $\theta$ is a global latent variable to model shared dependence.
In some HPPs, the distribution of $z_i$ is chosen to be a conjugate to $y_i$.
The eight schools model is one such example.
Figure \ref{fig:eight_schools_samples} shows samples of $\theta=[\mu,\tau]$ for the eight schools model obtained by both HMC and HMC-M.
Without marginalization, HMC struggles to sample small values of $\tau$ and the chain gets stuck due to the ``funnel'' relationship between $\tau$ and $x_i$~\cite{papaspiliopoulos2007general}.
With marginalization, both $x_i$ and the funnel are eliminated (from HMC), and the quality of the final samples significantly improves.
\begin{figure}
\centering
\includegraphics[width=0.23\textwidth, trim={1cm 1cm 1cm 1cm},clip]{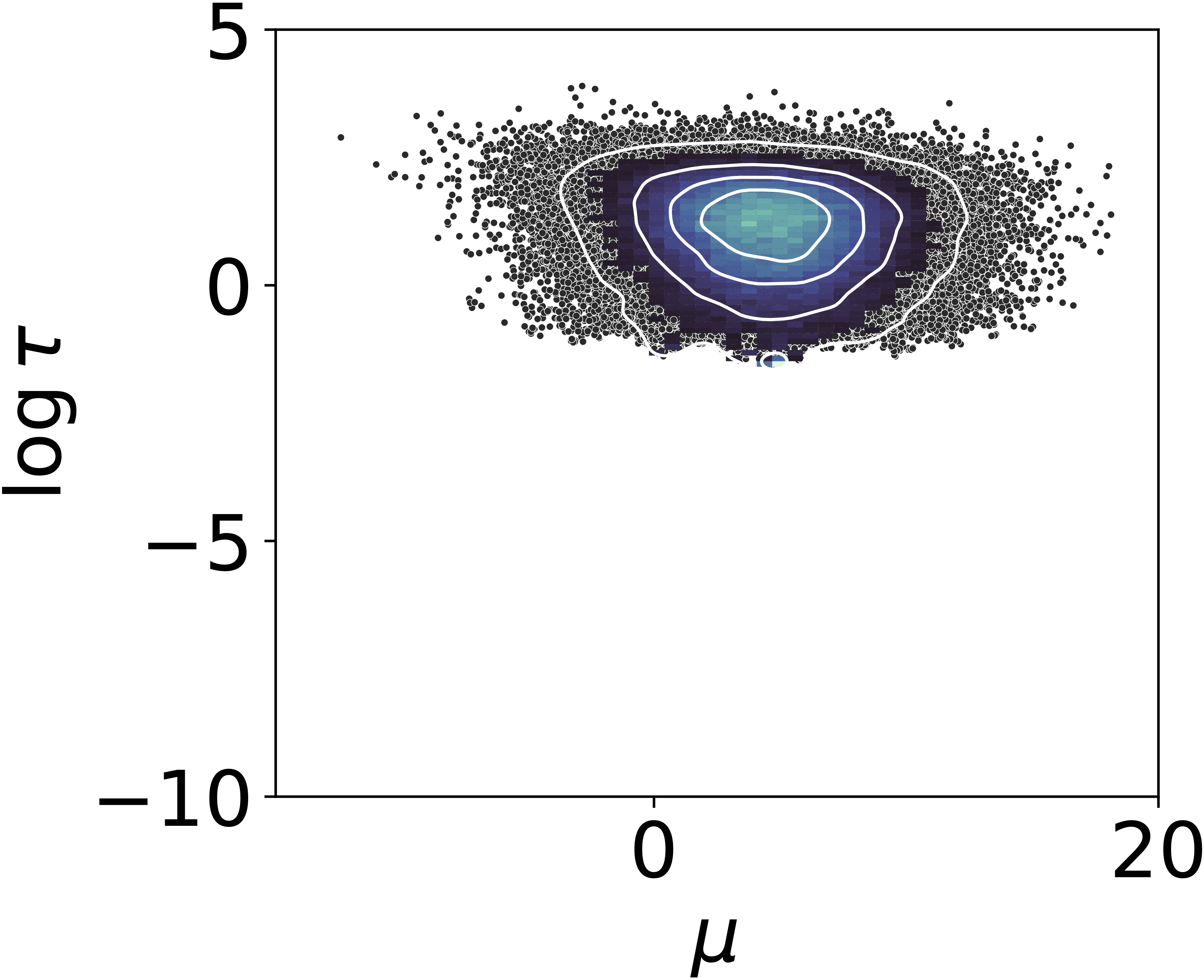}
\includegraphics[width=0.23\textwidth, trim={1cm 1cm 1cm 1cm},clip]{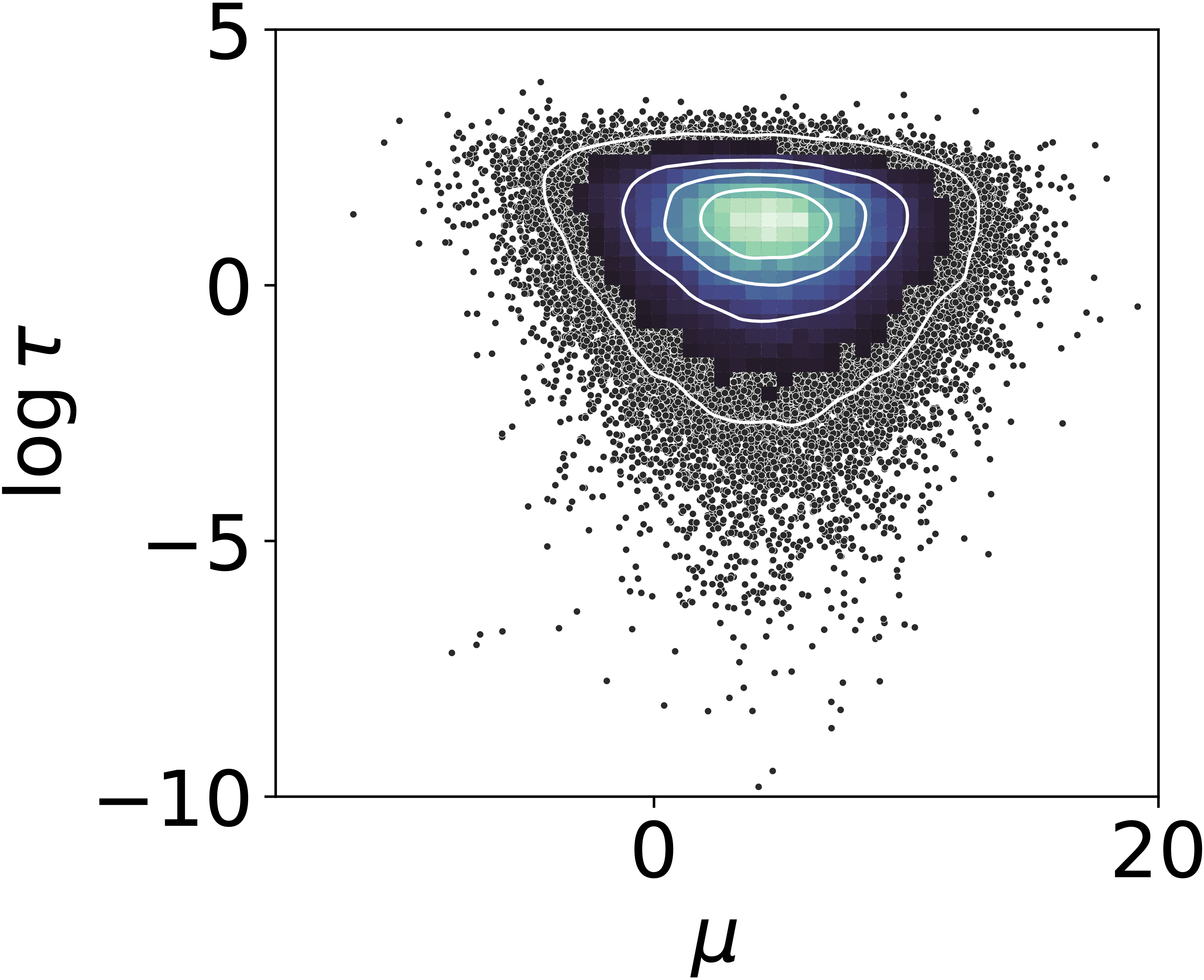}
\caption{Histogram and contour of samples from $\mu$ and $\log\tau$ from the eight schools model with different algorithms. 
In both cases, 100k samples were taken. 
Left: using HMC, with very low ESS, the chain fails to mix getting stuck in a particular region with almost no exploration of low values of $\tau$. 
Right: using HMC-M, the posterior is evenly explored.} 
\label{fig:eight_schools_samples}
\end{figure}

Another application of HPPs is repeated binary trials, where we observe the number of successes $y_i$ out of $K_i$ trials for each unit $i$, and assume a partially shared structure for the success probabilities, such as \cite{carpenter2017stan}:
\begin{gather}
	m\sim \text{Uniform}(0,1),\ \kappa\sim \text{Pareto}(1,1.5),\notag\\
	\theta_i\sim \text{Beta}(m\kappa,(1-m)\kappa),\ y_i\sim \text{Binomial}(K_i,\theta_i).\notag
\end{gather}
Applications include the rat tumors dataset~\cite{tarone1982use}, the baseball hits 1970 dataset~\cite{efron1975data} and the baseball hit 1996 AL dataset~\cite{carpenter2017stan}.
This model is again difficult for HMC due to a funnel relationship between $\kappa$ and $\theta_i$~\cite{carpenter2017stan}.
Suggested remedies are to model $\kappa$ with an exponential distribution \cite{patil2010pymc} or rewrite the model to one where reparameterization is applicable \cite{carpenter2017stan}.

We observe that, since $\theta_i$ (Beta) is locally conjugate to $y_i$ (Bernoulli), marginalization is a better strategy.
In the marginalized model, $y_i$ is a beta-binomial random variable, HMC samples only $m$ and $\kappa$, and each  $\theta_{i}$ is sampled afterward from $p(\theta_i \given m, \kappa, y_i)$, a beta distribution.
The funnel problem is eliminated and the HMC dimension is reduced from $n+2$ to 2.
Our methods achieve this automatically. 

Table \ref{table:hierarchical_partial_pooling} shows the results.
Sampling $\kappa$ is known to be difficult in this model, but HMC-M achieves an ESS with similar magnitude to the number of samples.
The HMC problem dimension is also reduced, which leads to faster running time.
These factors combined lead to more than 100x ESS/s improvement on the baseball hit 1996 AL data set. \edited{The same experiment is also conducted on Stan by manually rewriting the models using Algorithm \ref{algorithm1}. We confirm that Stan could also benefit from the same model transformations if we implemented our methods in that context. See Appendix \ref{sec:stan_res} for details.}
\begin{table*}[t]
\caption{Min ESS across all dimensions, time (s) and min ESS/s for HMC and HMC-M on the repeated binary trials model. Mean and std over 5 independent runs are reported. \javier{how many samples?}\jinlin{All the experiments are warming up for 10,000 steps and sampling for 100,000 steps, which is described in the first paragraph of the experiment section. }}
\label{table:hierarchical_partial_pooling}
\vskip 0.15in
\begin{center}
\begin{small}
\begin{sc}
\begin{tabular}{ccccc}
\toprule
Dataset &Algorithm& Min ESS & Time (s) & Min ESS/s\\
\midrule
\multirow{ 2}{*}{Baseball hits 1970 ($n=18$)}&HMC&1384.1 (1156.7) & \textbf{94.5} (5.7) & 14.8 (12.7)\\
&HMC-M&\textbf{39001.8} (20030.4) & 110.5 (89.2) & \textbf{592.3} (304.2)\\
\midrule
\multirow{ 2}{*}{Rat tumors ($n=71$)}&HMC&24632.3 (1494.5) & 654.8 (43.9) & 37.7 (2.1)\\
&HMC-M&\textbf{77644.5} (9570.8) & \textbf{72.4} (0.3) & \textbf{1072.7} (134.0)\\
\midrule
\multirow{ 2}{*}{Baseball hits 1996 AL ($n=308$)}&HMC&9592.3 (260.1) & 2746.1 (107.6) & 3.5 (0.2)\\
&HMC-M&\textbf{61109.0} (3344.9) & \textbf{130.9} (1.4) & \textbf{467.0} (29.5)\\
\bottomrule
\end{tabular}
\end{sc}
\end{small}
\end{center}
\vskip -0.1in
\end{table*}

\subsection{Hierarchical linear regression}
\label{exp:hlr}
Similar to partial pooling, hierarchy can be introduced in linear regression models.
We demonstrate with two examples how our methods can improve inference in such models.

\textbf{Electric company:}
The electric company model \cite{gelman2006data} studies the effect of an educational TV program on children's reading abilities.
There are $C=192$ classes in $G=4$ grades divided into $P=96$ treatment-control pairs.
Class $k$ is represented by $(g_k,p_k,t_k,y_k)$ where $g_k$ is the grade, $p_k$ is the index of pair, $t_k\in\{0,1\}$ is the treatment variable and $y_k$ is the average score.
The classes in pair $j$ belong to grade $\text{gp}[j]$.
The full model is
\begin{gather}
	\mu_{i}\sim\mathcal{N}(0,1),\ a_{j}\sim\mathcal{N}(100\mu_{\text{gp}[j]},1),\notag\\
	b_i\sim\mathcal{N}(0,100^2),\ \log\sigma_i\sim\mathcal{N}(0,1),\notag\\
	y_k\sim\mathcal{N}(a_{p_k}+t_kb_{g_k},\sigma_{g_k}^2).\notag
\end{gather}
where $i\in\{1,\dots,G\}$, $j\in\{1,\dots,P\}$ and $k\in\{1,\dots,C\}$.
Observe that $\mu_i$, $a_j$, $b_i$ and $y_k$ are all normally distributed with affine dependencies.
Therefore, it is possible to marginalize $\mu_i$, $a_j$ and $b_i$ from the HMC process.
As Section \ref{sec:small_electric_company} points out, it is very difficult to manually do so, but 
our methods do so automatically.

We observed that marginalization of $\mu_i$ led to very high JAX compilation times even though the computation graph for the log-density was not much larger than the one before marginalization (14606 primitive operations vs.\ 9186). 
We attribute this to a current JAX limitation.
See Appendix \ref{sec:slow_compilation} for experimental evidence.
As a workaround, we manually prevented $\mu_i$ from being marginalized.

Figure \ref{fig:electric_company} shows the results.
In this model, HMC on the original model performs poorly, but reparameterizing the $a_j$ variables is very helpful: this alternative is shown as HMC-R, and achieves excellent ESS (comparable to the number of samples). \edited{However, it is essential to note that it requires user intervention, either by employing the tools provided in \citet{gorinova2018automatic} or by implementing the reparameterization. It is also crucial to recognize that non-centered reparameterization may not consistently enhance the model and could, in some cases, prove detrimental \citep{gorinova2018automatic}. }
In addition, HMC-R does not reduce the dimension and solves a $3G+P=108$ dimension problem, while automatic marginalization reduces the problem dimension to $8$ and results in an additional 4x speed up.
\begin{figure}
\centering
\includegraphics[width=0.4\textwidth]{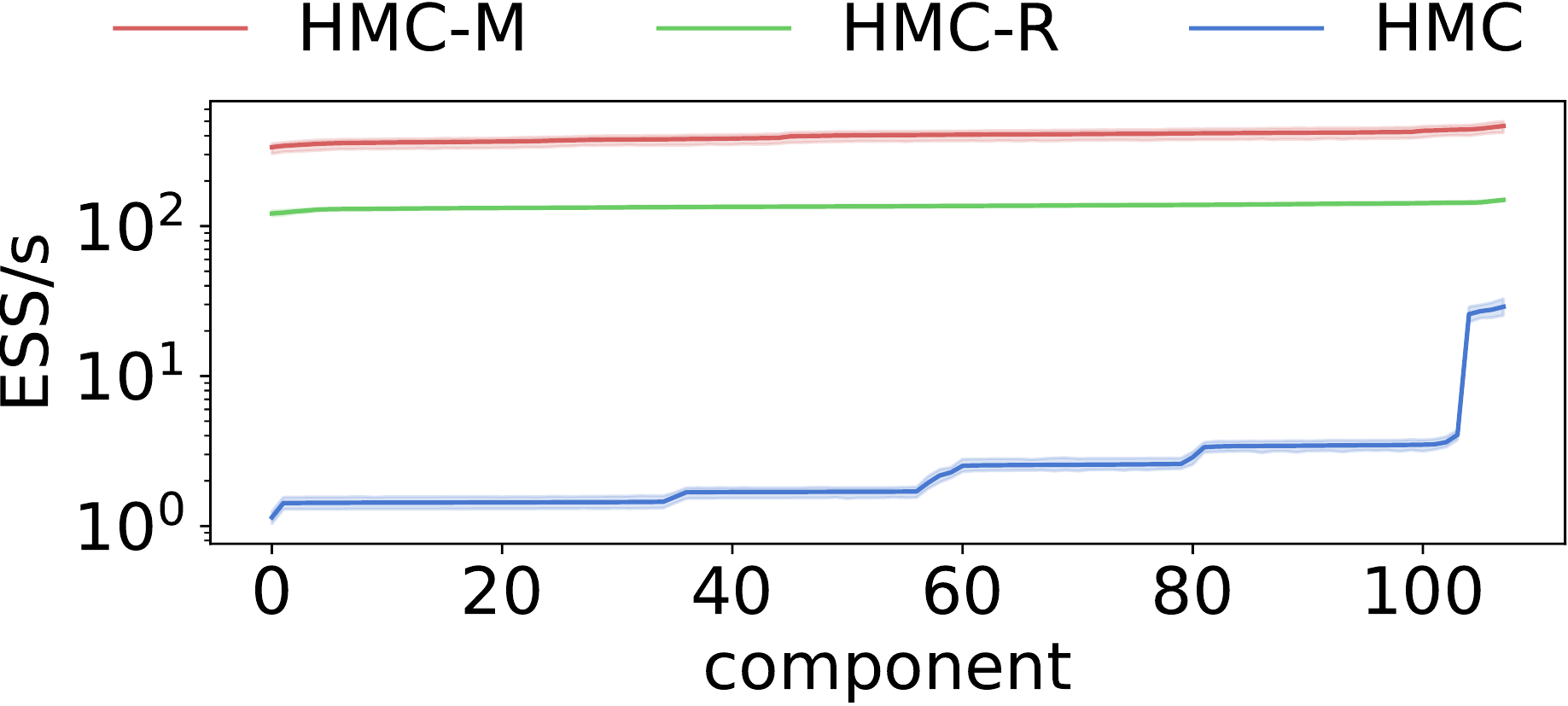}	
\caption{Component-wise ESS/s on the electric company model. Our method (HMC-M) is compared against HMC and HMC with reparameterization (HMC-R). Components ordered by ESS/s. }
\label{fig:electric_company}
\end{figure}

\textbf{Pulmonary fibrosis:} The Pulmonary fibrosis dataset \cite{osic-pulmonary-fibrosis-progression}
has patient observation records over time of forced vital capacity (FVC), a disease indicator.
The FVC of each patient is assumed to be linear with respect to time and regression slopes and intercepts are generated by a hierarchical model. \edited{Records form the pulmonary fibrosis dataset \cite{osic-pulmonary-fibrosis-progression} have the form $(\text{ID}_i,t_i,y_i)$, where $\text{ID}_i$ is the patient id, $t_i$ is the observation time, and $y_i$ is forced vital capacity (FVC), a measure of disease progression.
The FVC of each patient is assumed to be linear with respect to time and regression slopes and intercepts are generated by the following hierarchical model:
\begin{gather}
	\mu_\alpha\sim\mathcal{N}(0,500^2),\ \sigma_\alpha\sim \HalfCauchy(100),\notag\\
	\mu_\beta\sim\mathcal{N}(0,3^2),\ \sigma_\beta\sim \HalfCauchy(3),\notag\\
	\alpha_j\sim\mathcal{N}(\mu_\alpha,\sigma_\beta^2),\ \beta_j\sim\mathcal{N}(\mu_\beta,\sigma_\beta^2),\notag\\
	\sigma\sim\HalfCauchy(100),\ y_i\sim\mathcal{N}(\alpha_{\text{ID}_i}+t_i\beta_{\text{ID}_i},\sigma^2),\notag
\end{gather}
where $i\in\{1,\ldots,549\}$ and $j\in\{1,\ldots,173\}$.}
We again prevent two top-level variables from being marginalized and remove $\frac{2}{3}$ of the data points due to slow JAX compilation.
Under the settings, HMC-M outperforms HMC and HMC-R by producing more effective samples in less time (Table \ref{table:pulmonary_fibrosis}).

\begin{table}[t]
\caption{Min ESS across all dimensions, time (min) and min ESS/s for HMC, HMC-R and HMC-M on the pulmonary fibrosis model. Mean and std over 5 independent runs are reported.}
\label{table:pulmonary_fibrosis}
\vskip 0.15in
\begin{center}
\begin{small}
\begin{sc}
\begin{tabular}{cccc}
\toprule
Algorithm& Min ESS & Time (min) & Min ESS/s\\
\midrule
HMC&19817 (1207) & 51.8 (0.1) & 0.1 (0.0)\\
HMC-R&11362 (1567) & 51.5 (1.3) & 0.1 (0.0)\\
HMC-M&\textbf{96135} (557) & \textbf{14.9} (1.3) & \textbf{1.8} (0.1)\\
\bottomrule
\end{tabular}
\end{sc}
\end{small}
\end{center}
\vskip -0.1in
\end{table} 
\section{Discussion}
We proposed a framework to automatically marginalize variables in a graphical model obtained from a PPL for the purpose of accelerating MCMC.
Our results show significant performance improvements in models with conjugacy.
The process can be fully automated to free users from cumbersome derivations and implementations.
The method is limited to graphical models, which excludes some PPL features but covers a huge range of applied statistical models.
Our current implementation is limited to scalar and elementwise array operations.
An important direction for future work is to support a wider range of array operations, including matrix operations, indexing, slicing, and broadcasting.
Another interesting direction is to introduce automatic marginalization to MCMC applicable to higher order PPLs. We believe this work is an important step towards an automatic MCMC system that performs well on a wide range of models with minimal input from users.

\section*{Acknowledgements}
We thank Justin Domke and the anonymous reviewers for comments that greatly improved the manuscript. This material is based upon work supported by the National Science Foundation under Grant Nos. 1749854 and 1908577.
\bibliography{reference}
\bibliographystyle{icml2023}

\newpage
\appendix
\onecolumn
\javier{The listing for Algorithm~\ref{alg:full_reversing} should be in the first or, at most, second section of the Appendix. 
It is very important.} \jinlin{I think all parts are important. }

\section{Proof of the correctness of Definition \ref{def:edge_reversal}}
\label{proof:same_joint}
\addtocounter{definition}{-2}
Definition \ref{def:edge_reversal} in the main text is
\begin{definition}[Edge reversal]
  Assume $G$ is a graphical model where node $v$ is a parent of $c$ and there is no other path from $v$ to $c$. Reversing the $v \to c$ edge replaces factors $p(x_v \given  \x_{\pa(v)}) p(x_c \given  x_v, \x_{\pa(c)\setminus \{v\}})$ by $p(x_c \given  \x_U) p(x_v\given x_c, \x_U)$ and updates the parent sets as $\pa'(c) = U$, $\pa'(v) = U \cup \{c\}$, where $U=\pa(v)\cup \pa(c)\setminus\{v\}$.
\end{definition}
It was claimed that the operation yields a graphical model with the same joint distribution as the original. Now we prove it.
\begin{proof}
	It is enough to show that (1) the graphical model after reversal is still valid; (2) the joint distribution does not change, which requires
	\begin{align}
		p(x_v \given  \x_{\pa(v)}) p(x_c \given  x_v, \x_{\pa(c)\setminus \{v\}})=p(x_c \given  \x_U) p(x_v\given x_c, \x_U).\notag
	\end{align}
	For (1), we need to show that no cycles could be formed during the process. For any $p_v\in\pa(v)$, an edge $p_v\to c$ is added. Because there does not exist a path from $c$ to $p_v$ (otherwise there will be a loop in the original model), this edge will not cause a loop. For any $p_c\in\pa(c)\setminus \{v\}$, an edge $p_c\to v$ is introduced. This edge will also not cause a loop; otherwise another path from $v$ to $c$ will be found. Finally, the edge $v\to c$ is replaced with $c\to v$, this edge will also not introduce a loop because there are no other paths from $v$ to $c$. \dan{Maybe rename $s \to p_v$ and $t \to p_c$ to match the figure in the main text?}
	
	Now we show (2). Since there is no other paths from $v$ to $c$, there is no path from $v$ to any nodes in $\x_{\pa(c)\setminus \{v\}}$. Conditioned on $\x_{\pa(v)}$, by conditional independence, we have that $p(x_v \given  \x_{\pa(v)})=p(x_v \given  \x_{U})$. Also, all paths from nodes in $\pa(v)$ to $c$ are blocked either by $v$, or by a parent of $c$, so conditioned on $\x_{\pa(c)\setminus \{v\}}$, by independence, $p(x_c \given  x_v, \x_{\pa(c)\setminus \{v\}})=p(x_c \given  x_v, \x_{U})$. By the properties of conjugacy, we have that
	\begin{align}
		p(x_v \given  \x_{\pa(v)}) p(x_c \given  x_v, \x_{\pa(c)\setminus \{v\}})&=p(x_v \given  \x_{U})p(x_c \given  x_v, \x_{U})\notag\\
		&=p(x_c \given  \x_U) p(x_v\given x_c, \x_U).\notag
	\end{align}
        \dan{Where is conditional conjugacy used? I wasn't expecting it. The last equality? Isn't that immediate from the chain rule, i.e. $p(x_v | \x_U)p(x_c | x_v, \x_U) = p(x_v, x_c | \x_U) = p(x_c | \x_U) p(x_v | x_c, \x_U)$}
\end{proof}

\section{Proof of the Theorem 1}
\label{proof:th1}
We first restate Theorem \ref{thm:theorem1}. \dan{This was slightly edited in main text, should update}
\addtocounter{theorem}{-1}
\begin{theorem}
Let $G$ be a graphical model where node $v$ has children $c_1,\dots,c_H$.
If $x_v$ is locally conjugate to each of $x_{c_1},\dots,x_{c_H}$, then node $v$ can be turned into a leaf by sorting $c_1, \ldots, c_H$ by any topological ordering and reversing the edges from $v$ to each child following this ordering.
\end{theorem}

\begin{proof}
	Without loss of generality, assume $c_1,\ldots,c_H$ are sorted according to topological ordering. We prove by induction. Assume  for $k\in\{0,\ldots,H\}$, we have reversed the edges $v\to c_1,\ldots,v\to c_k$, and have the following properties:
	\javier{It would be clearer if these conditions are named and referred by named later in the proof. (It is going to be a little more than just adding an index like (i).)}
	\begin{enumerate}[label={(\arabic*)}]
		\item The children of $v$ are $c_{k+1},\ldots,c_H$.
		\item $c_{k+1}, \ldots, c_H$ are ordered topologically; \dan{``does not change''is not a property, it's a claim that some property is invariant. I think the property should be that $c_{k+1}, \ldots, c_H$ are ordered topologically.}
		\item $x_v$ is a local conjugate prior for each of $x_{c_{k+1}},\ldots,x_{c_{H}}$.
	\end{enumerate}
	We show that the edge $v\to c_{k+1}$ is reversible and the properties still hold for $k+1$ after the reversal. By (1) and (2), $c_{k+1}$ is minimal among the children of $v$ in topological order, so there does not exist a path from $v$ to $c_{k+1}$ other than $v\to c_{k+1}$.
	By (3), $x_v$ is a local conjugate prior to $x_{c_{k+1}}$. Then we can apply edge reversal to $v\to c_{k+1}$. Now we check all the conditions after the replacement. For property (1), the children of $v$ now become $c_{k+2},\ldots,c_H$. For property (2), the nodes with edges that changed (either incoming or outgoing) were $v$, $c_{k+1}$, and their parents; these nodes all preceded $c_{k+2}, \ldots, c_{H}$ in topological order prior to the reversal and continue to do so afterward, so the relative ordering of $c_{k+2}, \ldots, c_H$ does not change. \dan{I edited last sentence}
	\javier{This is using the base case that is in the body of the text: ``Suppose a node $v$ has only one child $c$ \dots''
	Now, that base case should be named and referred from here explicitly. }
	For (3), the distribution family of $x_v$ does not change, and the conditional distribution of each of $x_{c_{k+2}},\ldots x_{c_H}$ does not change. So all conditions of local conjugacy in Table \ref{table:conjugacy_pattern} will not change for them, which means the distribution of $x_v$ is still a local conjugate prior for the distributions of each of $x_{c_{k+2}},\ldots x_{c_H}$.
	
	In summary, the three conditions holds for $k=H$ by induction, which means $v$ can be converted to a leaf following the said procedure.
\end{proof}
\section{Details of conjugacy detection}
\label{sec:detail_conjugacy}
Conjugacy detection requires looking into the computation graph of functions. In this section, we introduce how conjugacy detection is performed on a computation graph. During the tracing of a program, the procedure of computation is compiled into intermediate representations consisting of basic operations. For example, the function
\begin{lstlisting}[language=Python]
def f(x, y):
    p = (x - y) ** 2
    q = (x + y) ** 2
    return p + q
\end{lstlisting}
could be represented as
\begin{lstlisting}[morekeywords={INPUTS,SUB,SQUARE,ADD,OUTPUTS}]
INPUTS: a, b
c = SUB a b
d = SQUARE c
e = ADD a b
f = SQUARE e
g = ADD d f
OUTPUTS: g
\end{lstlisting}
By looking at the intermediate representations, it is possible to reason about the relationship between outputs and inputs. For a function $f(x_{i_1},x_{i_2},\ldots,x_{i_k})$, we defined its computation graph to be a sequence of $N_f$ \emph{primitive operations} in Section \ref{sec:computation_graph}. The sequence $w_1,w_2,\ldots,w_{k+N_f}$ is computed, where: (1) the first $k$ are the inputs to the function, i.e., $w_j=x_{i_j}$ for $j$ = $1$ to $k$, and (2) each subsequent value is computed from the preceding values as
\begin{align}
	w_j=\phi_j(\mathbf w_{\pred(j)}),
\end{align}
where $\phi_j$ is a primitive operation on the set of values $\mathbf w_{\pred(j)}$, where $\pred(j)\subseteq\{1,\ldots,j-1\}$ is the set of predecessors of $j$. 
In Section \ref{sec:conjugacy_detection}, we have reduced conjugacy detection to affinity, linearity and dependency detections. We first introduce the details of dependency detection with the above definition. 

Given a function with a computation graph, we may want to determine whether a variable $w_{j}$ depends on an input $x_a$. We define the result to be $\DEPENDENT(w_j,x_a)$, which could be obtained recursively through the equations shown in Algorithm \ref{alg:dependent}.
\javier{I do not understand the case of line 5 in \ref{alg:dependent}. What is $x_A$? What is $\pred$? Is $j$ an index or more like $x$?}
\begin{algorithm}[t]
   \caption{Determining dependency of a variable on an input. $\x_{1:M}$ is the set of all random variables.}
   \label{alg:dependent}
\begin{algorithmic}[1]
\STATE  \textbf{function} DEPENDENT ($w_j$, $x$)
       \begin{ALC@g}
			\IF {$w_j=x$}
				\STATE \textbf{return} True
			\ENDIF
			\IF {$w_j\in \x_{1:M}$}
				\STATE \textbf{return} False
			\ENDIF
			\FOR {$p\in \pred(j)$}
				\IF{DEPENDENT($w_p$, $x$)}
					\STATE \textbf{return} True
				\ENDIF
			\ENDFOR
			\STATE \textbf{return} False
       \end{ALC@g}
\end{algorithmic}
\end{algorithm}
Note that in this paper, the inputs of functions are always random variables in $\x_{1:M}$. For $\DEPENDENT(w_j,x_a)$, if $w_j$ is a random variable in $\x_{1:M}$, then it must be an input of the function. Then $\DEPENDENT$ would return whether $w_j$ is $x_a$. If $w_j$ is the result of a basic operator, it would enumerate the inputs of that operator. If any of those inputs are dependent on $x_a$, then the result is true. The recursive algorithm could be of exponential complexity in some special cases. We store intermediate results in a dictionary and refer to it before recursion to avoid redundant computation. Then the complexity is linear with respect to the number of variables involved.
\dan{Could you say explicitly that the resulting algorithm is not exponential? Is it linear?}

Affinity and linearity can be included in the same framework. In addition to determining whether $w_j$ is affine to $x_a$, we further return whether the slope and intercept are non-zero. If affinity is detected with zero intercept, the relationship is then linear. We define $\text{AFFINE\_ALL}(w_j,x_a)$ to be a tuple of three bool variables - whether $w_j$ is affine on $x_a$, whether the slope is non-zero and whether the intercept is non-zero. Then $\LINEAR$ and $\AFFINE$ could be obtained from $\text{AFFINE\_ALL}(w_j,x_a)$. Our algorithm of affinity detection is adapted from \citet{AtkinsonSEMI22} with slight modification to setting that has no concrete values. The pseudocodes are in Algorithm \ref{alg:affinity}. The result of $\text{AFFINE\_ALL}(w_j,x_a)$ could be obtained by enumeration of cases of $\phi_j$ and induction in the structure of the computation graph.
\javier{that is, by induction in the structure of the computation graph.}
For example, if we know $w_j=\text{ADD}(w_{p_1},w_{p_2})$, and $r_1,s_1,t_1=\AFFINEALL(w_{p_1},x_a)$ and $r_2,s_2,t_2=\AFFINEALL(w_{p_2},x_a)$, then $w_j$ is affine to $x_a$ if both of $w_{p_1}$ and $w_{p_2}$ are affine to $x_1$, which means $(r_1\textbf{ and }r_2)$. The slope is non-zero if any of the slope of $p_1$ or $p_2$ is non-zero, so the second return value is $(s_1\textbf{ or }s_2)$. The same applies to whether the intercept is non-zero, which is $(t_1\textbf{ or }t_2)$. \javier{I couldn't understand this last part. I was expecting to see the variables $r_1, s_1, t_1$ and $r_2, s_2, t_2$ being used. } 

\begin{algorithm}[t]
   \caption{Determining affinity and linearity of a variable on an input. $\x_{1:M}$ is the set of all random variables.}
   \label{alg:affinity}
\begin{algorithmic}[1]
\STATE  \textbf{function} AFFINE\_ALL ($w_j$, $x$)
       \begin{ALC@g}
			\IF {$w_j=x$}
				\STATE \textbf{return} True, True, False
			\ENDIF
			\IF {$w_j\in \x_{1:M}$}
				\STATE \textbf{return} True, False, True
			\ENDIF
			\IF {$\phi_j\in\{\text{ADD},\text{SUB}\}$ \AND $\pred(j)=\{p_1,p_2\}$}
				\STATE $r_1,s_2,t_1=\AFFINEALL(w_{p_1},x)$
				\STATE $r_2,s_2,t_2=\AFFINEALL(w_{p_2},x)$
				\STATE \textbf{return} $r_1$ \AND $r_2$, $s_1$ \OR $s_2$, $t_1$ \OR $t_2$
			\ENDIF
			\IF {$\phi_j=\text{MUL}$ \AND $\pred(j)=\{p_1,p_2\}$}
				\STATE $r_1,s_2,t_1=\AFFINEALL(w_{p_1},x)$
				\STATE $r_2,s_2,t_2=\AFFINEALL(w_{p_2},x)$
				\IF {\textbf{not} $s_1$}
					\STATE \textbf{return} $r_1$ \AND $r_2$, $t_1$ \AND $s_2$, $t_1$ \AND $t_2$
				\ENDIF
				\IF {\textbf{not} $s_2$}
					\STATE \textbf{return} $r_1$ \AND $r_2$, $s_1$ \AND $t_2$, $t_1$ \AND $t_2$
				\ENDIF
				\STATE \textbf{return} False, False, False
			\ENDIF
			\IF {$\phi_j=\text{DIV}$ \AND $\pred(j)=\{p_1,p_2\}$}
				\STATE $r_1,s_2,t_1=\AFFINEALL(w_{p_1},x)$
				\STATE $r_2,s_2,t_2=\AFFINEALL(w_{p_2},x)$
				\IF {\textbf{not} $s_2$}
					\STATE \textbf{return} $r_1$ \AND $r_2$, $s_1$, $t_1$
				\ENDIF
				\STATE \textbf{return} False, False, False
			\ENDIF
			\FOR {$p\in \pred(j)$}
				\STATE $r,s,t=\AFFINEALL(w_p,x)$
				\IF{\textbf{not} $r$ \OR $s$}
					\STATE \textbf{return} False, False, False
				\ENDIF
			\ENDFOR
			\STATE \textbf{return} True, False, True
       \end{ALC@g}
\STATE
\STATE  \textbf{function} AFFINE ($w_j$, $x$)
       \begin{ALC@g}
			\STATE $r,s,t=\AFFINEALL(w_j,x)$
			\STATE \textbf{return} $r$
       \end{ALC@g}
\STATE
\STATE  \textbf{function} LINEAR ($w_j$, $x$)
       \begin{ALC@g}
			\STATE $r,s,t=\AFFINEALL(w_j,x)$
			\STATE \textbf{return} $r$ \AND \textbf{not} $t$
       \end{ALC@g}

\end{algorithmic}
\end{algorithm}
\section{Details of reversing an edge}
\label{sec:detail_reversing}
We have constructed the parts of conjugacy detection. In this section we discuss the details of reversing an edge in the marginalized MCMC. In Algorithm \ref{algorithm2} a function $\AFFINECOEFF$ is defined to get the coefficients when affinity is detected. The pseudocode of it is in Algorithm \ref{alg:reversing}, which is similar to $\AFFINE$. We emphasize that the computations in Algorithm \ref{alg:reversing} are fully symbolic. Each variable corresponds to a sequence of operations which could be regarded as a computation graph. So each $+$, $-$, $*$ and $/$ is applied as merging two (potentially overlapping) computation graphs. One issue is we need to define whether some variables are zero (lines 17,19,25). So operations of zeros should be specially dealt with. For example, if we find a $0+0$, instead of declaring an operation that adds two zeros, we should instead use the result $0$. 

\begin{algorithm}[t]
   \caption{Getting the coefficients of affine relationship between a variable $w_j$ on an input $x$. $\x_{1:M}$ is the set of all random variables.}
   \label{alg:reversing}
\begin{algorithmic}[1]
\STATE  \textbf{function} \AFFINECOEFF ($w_j$, $x$)
       \begin{ALC@g}
			\IF {$w_j=x$}
				\STATE \textbf{return} $1$, $0$
			\ENDIF
			\IF {$w_j\in \x_{1:M}$}
				\STATE \textbf{return} $0$, $1$
			\ENDIF
			\IF {$\phi_j=\text{ADD}$ \AND $\pred(j)=\{p_1,p_2\}$}
				\STATE $s_1,t_1=\AFFINECOEFF(w_{p_1},x)$
				\STATE $s_2,t_2=\AFFINECOEFF(w_{p_2},x)$
				\STATE \textbf{return} $s_1 + s_2$, $t_1 + t_2$
			\ENDIF
			\IF {$\phi_j=\text{SUB}$ \AND $\pred(j)=\{p_1,p_2\}$}
				\STATE $s_1,t_1=\AFFINECOEFF(w_{p_1},x)$
				\STATE $s_2,t_2=\AFFINECOEFF(w_{p_2},x)$
				\STATE \textbf{return} $s_1 - s_2$, $t_1 - t_2$
			\ENDIF
			\IF {$\phi_j=\text{MUL}$ \AND $\pred(j)=\{p_1,p_2\}$}
				\STATE $s_1,t_1=\AFFINECOEFF(w_{p_1},x)$
				\STATE $s_2,t_2=\AFFINECOEFF(w_{p_2},x)$
				\IF {$s_1$ is 0}
					\STATE \textbf{return} $t_1*s_2$, $t_1*t_2$
				\ENDIF
				\IF {$s_2$ is 0}
					\STATE \textbf{return} $s_1*t_2$, $t_1*t_2$
				\ENDIF
				\STATE \textbf{raise} Error
			\ENDIF
			\IF {$\phi_j=\text{DIV}$ \AND $\pred(j)=\{p_1,p_2\}$}
				\STATE $s_1,t_1=\AFFINECOEFF(w_{p_1},x)$
				\STATE $s_2,t_2=\AFFINECOEFF(w_{p_2},x)$
				\IF {$s_2$ is 0}
					\STATE \textbf{return} $s_1/t_2$, $t_1/t_2$
				\ENDIF
				\STATE \textbf{raise} Error
			\ENDIF
			\STATE \textbf{return} $0$, $w_j$
       \end{ALC@g}
   
\end{algorithmic}
\end{algorithm}
Now we are only left with the full version of Algorithm \ref{algorithm2}, which could be found in Algorithm \ref{alg:full_reversing}.

\begin{algorithm}[t]
   \caption{The full version of Algorithm \ref{algorithm2}: reversing an edge}
   \label{alg:full_reversing}
\begin{algorithmic}[1]
\STATE  \textbf{function} REVERSE ($G=(D_i,\pa(i),f_i)_{i=1}^M$,$v$,$c$)
       \begin{ALC@g}
       		\STATE // $\pa(v), \pa(c), f_v, f_c, D_c$ updated in place
                \STATE // all variables represent symbolic expressions\\[2pt]
			\IF {$D_c$ is normal \textbf{and} $D_v$ is normal}
			\STATE Let $\mu_v$ and $\sigma_v^2$ be the two output expresssions $f_v$, and $\mu_c$ and $\sigma^2_c$ be the output expressions of $f_c$.
			\STATE $p,q=\text{AFFINE\_COEFF}(\mu_c,x_v)$
			\STATE $k=\sigma_v^2p/(p^2\sigma_v^2+\sigma_c^2)$
			\STATE $\mu_c'=p\mu_v+q$
			\STATE $\sigma_c'^2 = p^2\sigma_v^2+\sigma_c^2$
			\STATE $\mu_v'=\mu_v+k(x_c-\mu_c')$
			\STATE $\sigma_v'^2 = (1-kp)\sigma_v^2$
			\STATE $f_c=(\mu_c',\sigma_c'^2)$, $f_v=(\mu_v',\sigma_v'^2)$
			\ENDIF
			\IF {$D_v$ is Beta \textbf{and} $D_c\in\{\text{Bernoulli},\text{Binomial}\}$}
			\STATE Let $\alpha_v$ and $\beta_v$ be the two output expressions of $f_v$
			\IF {$D_c=\text{Bernoulli}$}
			\STATE $n_c=1$
			\STATE $p_c=\lambda_c$
			\ELSE 
			\STATE Let $n_c$ and $p_c$ be the two output expressions of $f_c$
			\ENDIF
			\STATE $\alpha_v'=\alpha_v+x_c$
			\STATE $\beta_v'=\beta_v+n_c-x_c$
			\STATE $D_c=\text{BetaBinomial}$
			\STATE $f_c=(n_c,\alpha_v,\beta_v)$, $f_v=(\alpha_v',\beta_v')$
			\ENDIF
			\IF {$D_v$ is Gamma \textbf{and} $D_c\in\{\text{Exponential},\text{Gamma}\}$}
			\STATE Let $\alpha_v$ and $\beta_v$ be the two output expressions of $f_v$
			\IF {$D_c=\text{Exponential}$}
			\STATE $\alpha_c=1$
			\STATE $\beta_c=\lambda_c$
			\ELSE 
			\STATE Let $\alpha_c$ and $\beta_c$ be the two output expressions of $f_c$
			\ENDIF
			\STATE $p,q=\AFFINECOEFF(\beta_c,x_v)$
			\STATE $\alpha_v'=\alpha_v+\alpha_c$
			\STATE $\beta_v'=\beta_v+p*x_c$
			\STATE $D_c=\text{CompoundGamma}$
			\STATE $f_c=(\alpha_c,\alpha_v,\beta_v/p)$, $f_v=(\alpha_v',\beta_v')$
			\ENDIF
			\STATE $\pa(c)=(\pa(c)\setminus \{v\})\cup \pa(v)$	
			\STATE $\pa(v)=\pa(v)\cup \{c\}\cup \pa(c)$
			\STATE \textbf{return} $G$
       \end{ALC@g}
   
\end{algorithmic}
\end{algorithm}
\section{Relations to the hoisting algorithm}
\label{sec:relation_hoisting}
The hoisting algorithm in \citet{AtkinsonSEMI22} is an online algorithm that can be used for automatically running Rao-Blackwellized particle filters (RBPF) \cite{doucett2000rao}. Our Algorithm \ref{algorithm1} is highly related to the hoisting algorithm. Both algorithms can perform conjugacy detection and marginalize all possible random variables. One apparent difference is that Algorithm \ref{algorithm1} is written as loops while the hoisting algorithm uses recursions. However, the main difference between the two algorithms comes from the application. In RBPF, all non-marginalizable random variables are sampled from the model, allowing the representations to be semi-symbolic, where non-marginalizable random variables are replaced with sampled values during the execution of the hoisting algorithm. In HMC, no random variables are directly sampled, and the same computation graph will be executed many times, so marginalization should be performed before running with fully symbolic representations. In the mean time, Theorem \ref{thm:theorem1} allows us to separate the conjugacy detections and the reversals in two different loops, which reduces the running time in large scale models. This improvement is not possible with the hoisting algorithm as an online algorithm, so some unnecessary reversals are performed. Furthermore, from the perspective of implementation, the parent node is fixed inside the loop of $v$ in Algorithm \ref{algorithm1}. So in one iteration, all the calls to $\text{CONJUGATE}$ and $\text{REVERSE}$ are with respect to the same $v$. It is therefore possible to save the intermediate results of these functions to avoid redundant computation on the computation graph. So the time complexity of Algorithm \ref{algorithm1} is $O(M|C|)$, where $|C|$ is the size of the computation graph. With the above considerations, we think that Algorithm \ref{algorithm1} is an important contribution in the area. 
\edited{
\section{Stan results with manual marginalization}
\label{sec:stan_res}
\begin{table*}[t]
\caption{Min ESS across all dimensions, time (s) and min ESS/s for NUTS and NUTS with manual marginalization (NUTS-M) on the repeated binary trials model with Stan. Mean and std over 5 independent runs are reported.}
\label{table:stan_hierarchical_partial_pooling}
\vskip 0.15in
\begin{center}
\begin{small}
\begin{sc}
\begin{tabular}{ccccc}
\toprule
Dataset &Algorithm& Min ESS & Time (s) & Min ESS/s\\
\midrule
\multirow{ 2}{*}{Baseball hits 1970 ($n=18$)}&NUTS&3647.6 (1041.0) & 	23.3 (0.5) & 155.9 (43.4)\\
&NUTS-M&\textbf{29674.5} (3209.4) & 	\textbf{21.1} (0.3) &\textbf{1407.9} (142.1)\\
\midrule
\multirow{ 2}{*}{Rat tumors ($n=71$)}&NUTS&23556.4 (1487.4)&37.9 (1.4)	&623.7 (59.0)\\
&NUTS-M&\textbf{45241.9} (1445.4)&	\textbf{24.8} (0.4)&	\textbf{1826.2} (84.4)\\
\midrule
\multirow{ 2}{*}{Baseball hits 1996 AL ($n=308$)}&NUTS&9684.6 (529.8)&	99.8 (1.9)&	97.0 (5.6)\\
&NUTS-M&\textbf{42718.0} (1207.6)&	\textbf{43.0} (0.6)&	\textbf{994.1} (27.0)
\\
\bottomrule
\end{tabular}
\end{sc}
\end{small}
\end{center}
\end{table*}
Although automatic marginalization has not been implemented in Stan, it is possible to confirm its applicability by testing on Stan with manual marginalization. Specifically, it is possible to rewrite the model as the marginalized one and use the \texttt{generated quantities} section to recover the marginalized variables. We replicated the experiments in Table \ref{table:hierarchical_partial_pooling} and the results is in Table \ref{table:stan_hierarchical_partial_pooling}. By introducing marginalization, NUTS in Stan not only runs faster, but also gets better samples for the repeated binary trials model.
}
\section{Slow compilation of JAX}
\label{sec:slow_compilation}
In the experiments, we discover that the compilation time of JAX can be slow for some models. We identify the problem specifically at the structure of marginalized hierarchical models. To demonstrate, we consider the simple model
\begin{align}
	x\sim\mathcal{N}(0,1),\ \log\sigma\sim\mathcal{N}(0,1),\ y_i\sim\mathcal{N}(x,\sigma),\notag
\end{align}
where $i=1,\ldots,N$ and $y_i=0$ for all $i$ are provided as pseudo observations. It is possible to marginalize $x$ by reversing edges to each of $y_i$. However, we found that the JIT compilation time scales super-linear with respect to $N$ for the marginalized model. See Figure \ref{fig:compilation}.
\begin{figure}[t]
\centering
\includegraphics[width=.45\textwidth]{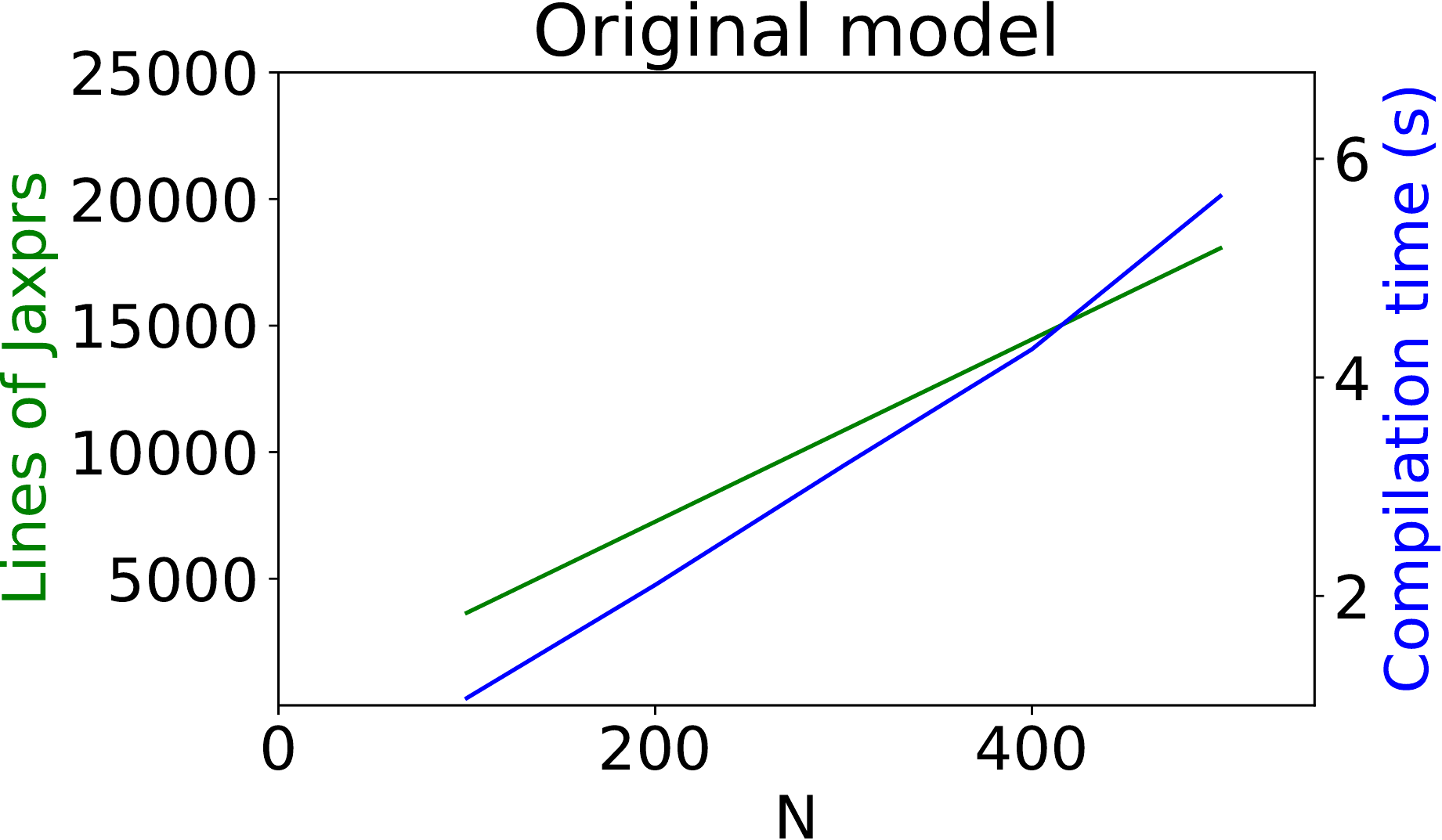}\hfill{}
\includegraphics[width=.45\textwidth]{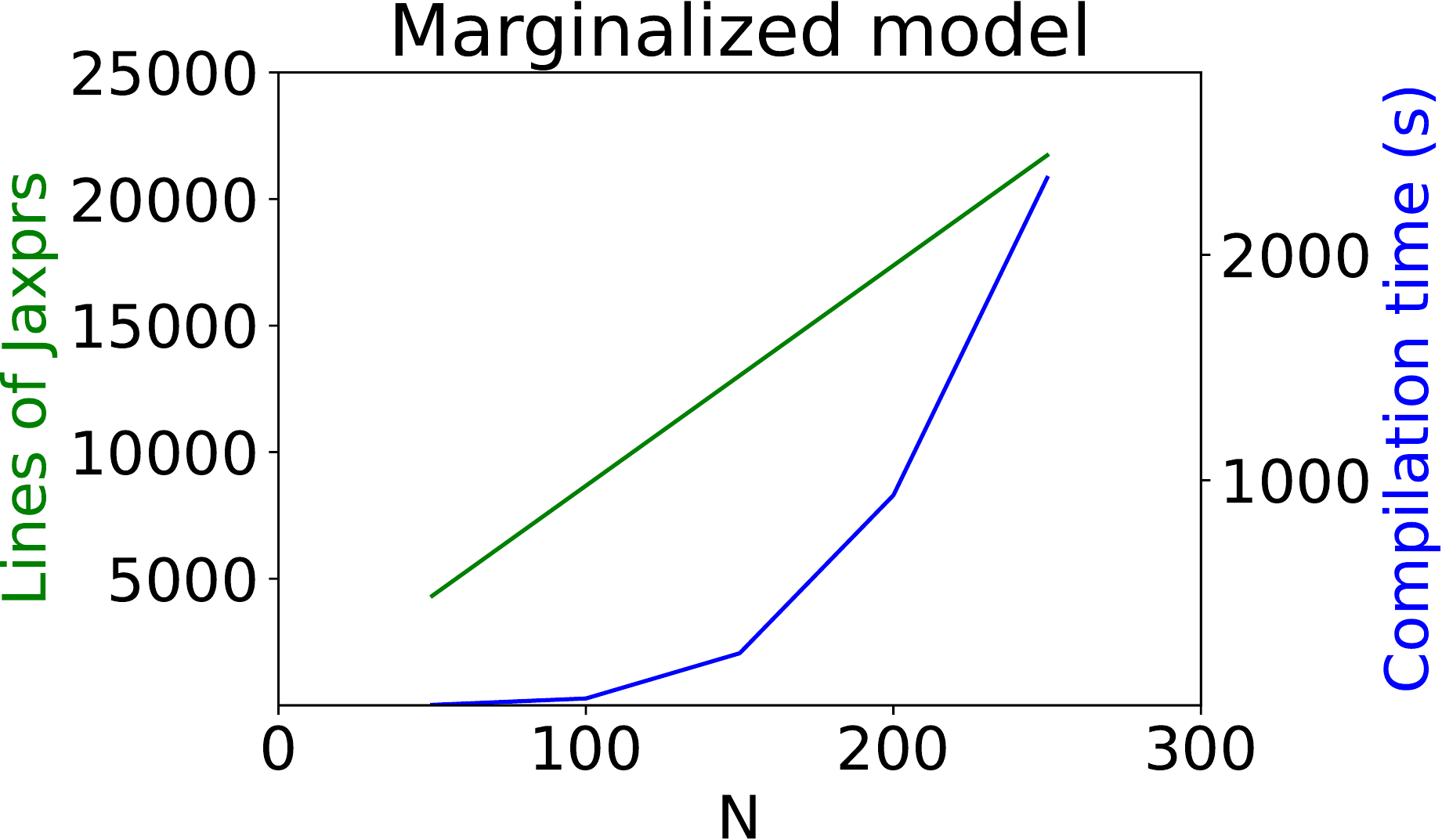}	
\caption{Lines of Jaxprs of the gradient of the log density function and JIT compilation time with respect to $N$ for the simple hierarchical model. With similar lines of Jaxprs, the compilation time can be hundreds of times slower on the marginalized model than on the original model.}
\label{fig:compilation}
\end{figure}
Regardless of the performance, the JIT compilation time for the gradient function of the marginalized model can be hundreds larger than that of the original model when $N$ is large enough, with similar lines of Jaxprs. This is probably because marginalization creates a chain shaped computation graph for all the observations, and it is difficult for JAX to work in this case. We do not regard it as a core limitation of our idea.

\end{document}